\documentclass[journal, 12pt, draftclsnofoot, onecolumn]{IEEEtran}
\normalsize
\usepackage{cite}
\usepackage{amsmath,amssymb,amsfonts,amsthm, amssymb}
\usepackage{algorithmic}
\usepackage{array,multirow,graphicx}
\usepackage{multirow}
\usepackage{textcomp}
\usepackage{xcolor}
\usepackage{mathtools}
\usepackage{comment}
\usepackage[ruled]{algorithm}
\usepackage{caption}
\usepackage{subcaption}
\usepackage{mathtools}
\usepackage{enumitem}
\usepackage{placeins}
\usepackage{url}
\usepackage{bbding}
\usepackage{tabularx}
\usepackage{hyperref}
\usepackage{pifont}

\usepackage{balance}
\newtheorem*{remark}{Remark}
\newtheorem{proposition}{Proposition}
\newtheorem{theorem}{Theorem}
\usepackage{color, colortbl}
\definecolor{LightCyan}{rgb}{0.88,1,1}
\usepackage[normalem]{ulem}
\newcolumntype{b}{>{\hsize=1.55\hsize\linewidth=\hsize}X}
\newcolumntype{s}{>{\hsize=.45\hsize}X}
\ifCLASSINFOpdf
\else
\fi
\begin{document}
\title{Federated Learning in UAV-Enhanced Networks: Joint Coverage and Convergence Time Optimization}
\author{Mariam~Yahya,~Setareh~Maghsudi,~and~Slawomir~Stanczak
}
\maketitle
\begin{abstract}
Federated learning (FL) involves several devices that collaboratively train a shared model without transferring their local data. FL reduces the communication overhead, making it a promising learning method in UAV-enhanced wireless networks with scarce energy resources. Despite the potential, implementing FL in UAV-enhanced networks is challenging, as conventional UAV placement methods that maximize coverage increase the FL delay significantly. Moreover, the uncertainty and lack of a priori information about crucial variables, such as channel quality, exacerbate the problem. In this paper, we first analyze the statistical characteristics of a UAV-enhanced wireless sensor network (WSN) with energy harvesting. We then develop a model and solution based on the multi-objective multi-armed bandit theory to maximize the network coverage while minimizing the FL delay. Besides, we propose another solution that is particularly useful with large action sets and strict energy constraints at the UAVs. Our proposal uses a scalarized best-arm identification algorithm to find the optimal arms that maximize the ratio of the expected reward to the expected energy cost by sequentially eliminating one or more arms in each round. Then, we derive the upper bound on the error probability of our multi-objective and cost-aware algorithm. Numerical results show the effectiveness of our approach.
\let\thefootnote\relax\footnotetext{Parts of this paper were presented at the European Wireless Conference, Dresden, Germany in September 2022 \cite{yahya_conf}. M. Yahya is with the Department of Computer Science, University of Tübingen, 72074 Tübingen, Germany (email: mariam.yahya@uni-tuebingen.de). S. Maghsudi is with the Department of Computer Science, University of Tübingen, 72074 Tübingen, Germany and the Fraunhofer Heinrich Hertz Institute, 10587  Berlin, Germany (email: setareh.maghsudi@uni-tuebingen.de). S. Stanczak is with the Department of Telecommunication Systems, The Technical University of Berlin, 10587 Berlin, Germany, and the Fraunhofer Heinrich Hertz Institute, 10587 Berlin, Germany (email: slawomir.stanczak@hhi.fraunhofer.de).} 
\end{abstract} 
%
\begin{IEEEkeywords}
Federated learning, unmanned aerial vehicles, coverage, convergence time, multi-objective multi-armed bandits.
\end{IEEEkeywords}
%
\section{Introduction}
The advances in the hardware for on-device intelligence enable unmanned aerial vehicles (UAVs) to offer computation-intensive services based on machine learning methods, such as image classification \cite{zhang2022}{\color{black} and smart agriculture  \cite{friha2021internet}}. Implementing centralized machine learning methods in UAV networks is often infeasible due to the requirement of transferring large data volumes and privacy hazards. To mitigate such shortcomings, federated learning (FL) has emerged as an alternative learning paradigm in UAV-enhanced networks.\\
FL is a distributed learning algorithm that preserves the clients' privacy and reduces communication costs by training a shared model mainly locally in the clients without sharing their raw data. In each global iteration, the server transmits the global model to the clients to update it using their local data. Then, they send the updated local models to the server that aggregates them to start a new iteration. This iterative process of transmitting and updating the models continues until achieving the desired global accuracy \cite{mcmahan2017communication}. FL facilitates various applications and services in UAV-enhanced networks. These include communication, network management, and computation offloading \cite{nguyen2021federated}. \\
One major challenge of FL is the potentially lengthy convergence, which stems from (i) the prolonged transmission time of the local and global updates over the wireless channels, (ii) the computation delays at the FL clients to update the local models, and (iii) the iterative global updates. There are different methods to address this issue, such as client selection 
\cite{chen2020convergence,Huang2021,xia2020multi,cho2020bandit,xia2021federated,9390199}, resource allocation (RA) \cite{FL_Dinh2021,FL_zeng2020}, and update compression \cite{becking2022adaptive}. Client selection is one of the most common approaches for reducing the FL convergence time. It selects the clients that lead to faster convergence to participate in FL, which makes it particularly useful in networks with many clients and a few channels. The CS criteria often depend on the client's communication- and computation delays in each global iteration \cite{xia2021federated}, the clients' local losses  \cite{cho2020bandit}, or their data type \cite{zhang2021client}. To ensure the clients' participation in FL training,  \cite{Huang2021,xia2020multi,xia2021federated,cho2020bandit} tackle the fairness issue in CS. Alternatively, \emph{Dinh et al} \cite{FL_Dinh2021} use the resource allocation approach to balance two conflicting objectives, minimizing the FL time and reducing the energy consumption. Convergence time minimization is a crucial research topic also in UAV-enhanced networks \cite{FL_zeng2020,yang2021privacy}. The authors in \cite{FL_zeng2020} design a joint power allocation and scheduling scheme for a UAV swarm to optimize the FL convergence while controlling energy consumption. Finally, asynchronous FL schemes shorten the convergence time \cite{yang2021privacy} as the servers do not wait for all local updates.\\
Another major challenge to optimizing the FL performance is the uncertainty due to the randomness and lack of information, for example, in channel qualities, resource availability, and network density or topology. Even though the channel state information and the variables above are unavailable in realistic scenarios, the majority of the cutting-edge research assumes otherwise \cite{FL_Dinh2021,FL_zeng2020,yang2021privacy}. Still, a few works use machine learning-based methods to overcome this challenge. The most relevant example to our work example is the multi-armed bandit (MAB) framework, where the network manager learns the network's statistical information sequentially during the FL training process \cite{xia2020multi, Huang2021,cho2020bandit,xia2021federated,9390199,9500804}. More details on the MAB framework is given in \textbf{Section~\ref{sec:intro_MOMAB}}. \textbf{Table~\ref{tab:SoA}} summarizes the most relevant state-of-the-research.
\begin{table}[!t]
\centering
    \caption{A comparison between the FL convergence time minimization approaches.}
  \begin{tabularx}{\textwidth}{l| l c c b s} \hline
    \rowcolor[gray]{0.8} &  Paper  & MAB & UAV & Highlights & \scriptsize{Other considerations} \\ \hline
    \multirow{6}{1em}{\rotatebox[origin=c]{90}{Client selection \hspace{4.5cm}}} &  \cite{xia2020multi}   & \checkmark & \ding{55} &  Minimizes the transmission and local computation time for two cases, clients with i.i.d and balanced datasets and non-i.i.d. and unbalanced datasets for which a fairness constraint is added.& Fairness in CS and the clients' availability. \\ \cline{2-6}
    & \cite{Huang2021}    & \checkmark  & \ding{55}  & Uses contextual combinatorial bandits (C$^2$MAB) to estimate the model exchange time between the clients and the server and balances between the fairness in CS and the FL training efficiency. &  Fairness in CS and the clients' availability. \\ \cline{2-6}
   & \cite{cho2020bandit}  & \checkmark  & \ding{55}  & Selects the clients with higher local losses for faster FL convergence using MABs and increases the fairness in user selection.& Fairness in CS. \\ \cline{2-6}
  & \cite{xia2021federated}  &  \checkmark  & \checkmark  & Reduces the number of communication rounds by selecting the clients based on their update staleness and weight divergence, then it shortens the average time interval per round by performing latency based client scheduling using MABs. & Fairness in CS.  \\ \cline{2-6}
  & \cite{9390199} &  \checkmark  & \ding{55} & Uses the $\epsilon$-greedy algorithm to select the clients that jointly minimize the the number of global rounds (clients with a significant local losses) and the delay per round (clients with lower latency). & -- \\ \cline{2-6}
 & \cite{yang2021privacy}  & \ding{55} & \checkmark & Aims for faster FL convergence. The UAV servers employ an asynchronous advantage actor-critic-based algorithm for joint device selection, UAV placement, and resource management. & The Reward captures the FL time and accuracy loss.  \\ \hline
{\multirow{2}{*}{\rotatebox[origin=c]{90}{ Resource allocation \hspace{1cm}}}} & \cite{FL_zeng2020} & \ding{55} & \checkmark  & Minimizes the number of global rounds by jointly designing the power allocation and scheduling in a UAV network. This work considers the transmission and UAV flying power. The optimization problem is solved using the sample average approximation and dual method.& The UAVs' stability and their flying and transmission power consumption \\ \cline{2-6}
& \cite{FL_Dinh2021} & \ding{55} & \ding{55}  & Proposes an FL algorithm named FEDL for heterogeneous clients and characterizes the trade-off between the local and gobal iterations.  It also proposes a resource allocation problem in wireless networks to achieve a trade-off between the FL time and the energy consumption. & The  heterogeniety in the clients datasize and CPU frequency.  \\ \hline
{\multirow{2}{*}{\rotatebox[origin=c]{90}{UAV coverage\hspace{1.2cm}} }} & \cite{yahya2022} & \ding{55} & \checkmark &  Minimizes the FL delay by optimizing the locations of the UAVs to achieve homogeneous coverage across the UAVs to minimize the maximum computation delay, followed by a bandwidth and power allocation step to minimize the communication delay.  & The power is allocated to achieve max-min fairness in the data rate.   \\  \cline{2-6}
& Ours & \checkmark & \checkmark  & Proposes a MO-MAB-based approach to jointly maximize the UAVs coverage and minimize the FL local model update time under uncertainty. It also proposes a solution that is particularly with a large number of arms. It is based on a scalarized-BAI algorithm.  & Uncertainty about the channel and the sensors' location and availability.      \\ \hline
    \end{tabularx}
    \label{tab:SoA}
\end{table}
{\color{black} Thanks to their flexibility and low cost, UAVs integrate well into 5G, 6G, and internet-of-things (IoT) networks as aerial base stations to improve their performance and enable new applications \cite{friha2021internet, FL_zeng2020,yahya2022,yang2021privacy,yang2021fresh,yang2021proactive,peng2023energy}. For example, UAVs are used in private content caching while ensuring age-of-information aware, fair, and energy-efficient content delivery \cite{yang2021fresh}. In such applications, ML methods are widely applicable. For example, deep neural networks predict the mobile users' locations to perform UAV-network slicing to guarantee coverage and the ultra-reliable low latency communications (URLLC) requirements of control signals \cite{yang2021proactive}. Moreover, using deep reinforcement learning, one can optimize the phase shifts of reconfigurable intelligent surfaces for energy harvesting \cite{peng2023energy} and the energy efficiency of blockchain operations in UAV-assisted computation offloading in IoT networks \cite{lan2023uav}. In this paper, we optimize the performance of a FL UAV-enhanced network under high uncertainty and conflicting objectives using MABs.
}
\subsection{Motivation}
As described above, the UAVs might engage in an FL procedure, each using the data collected locally in its coverage area \cite{FL_Brik2020}, {\color{black} one significant example appears in IoT for agriculture for crop monitoring, fertilizer management, and data collection \cite{friha2021internet}}.
In such a scenario, the UAVs' placement and coverage substantially affect the FL convergence time, an issue that remains unaddressed to a great extent. The conventional approaches of UAV placement for coverage maximization prolong the FL convergence delay drastically due to the inhomogeneous sensor- or user-density in the UAVs' coverage area. Larger data volumes often require more processing time, so the FL server waits longer for slower UAVs that cover dense locations to update the global model. Such a straggler effect increases the overall delay. In  \cite{yahya2022}, we proposed an approach to jointly maximize the network coverage and minimize the FL convergence time by effective UAVs' placement and radio resource management. The UAVs' placement ensured higher homogeneity in the number of sensors covered by each UAV, thus, similar processing times. \\
This paper proposes a novel method to jointly maximize the sensor coverage and minimize the FL convergence time in a UAV-enhanced network under uncertainty about the sensors' distribution and channels' quality. It also considers the non-deterministic sensors' activity due to opportunistic energy harvesting. The two objectives are conflicting since increasing the sensor coverage results in heavier computation and extended delays. To solve this problem under uncertainty, we use a multi-objective multi-armed bandits (MO-MAB) framework \cite{drugan2013} merged as an alternative learning paradigm in UAV networks. With this approach, we provide a solution for a realistic UAV network that is robust to information shortage.
\subsection{MO-MAB as a Solution Framework}
\label{sec:intro_MOMAB}
MAB is an online sequential decision-making problem where, in each round, an agent selects an arm from a set to receive a reward sampled from its apriori unknown reward-generating process. The agent intends to choose the arms to minimize its long-term cumulative regret \cite{Maghsudi16:MAB}. In the multi-objective version of the MAB problem, each arm returns a reward vector, where each element is the reward of the corresponding objective \cite{drugan2013}. Since there are multiple objectives, there can be several optimal reward vectors, thus, several optimal arms \cite{drugan2013}, \cite{ruadulescu2020}. Hence the agent aims to simultaneously minimize the regret in all objectives by playing optimal arms fairly.\\
\textit{Drugan et al.} proposed two ways to obtain the set of optimal arms in a MO-MAB problem: the Pareto partial order approach and the scalarization approach \cite{drugan2013}. The former maximizes the reward vectors directly in the multi-objective reward space by looking at the dominance relationships between arms, while the latter uses a set of appropriate scalarization functions to transform the multi-objective problem into a single objective one and find the optimal arms. In this paper, we solve our bi-objective problem of coverage maximization and convergence time minimization using both approaches and compare their performance. \\
One challenge in our problem formulation is the large number of arms. In addition, the scalarized multi-objective approach requires finding the optimal arm for each scalarization function. To address these issues, we develop an algorithmic solution to our bi-objective optimization problem based on the best-arm identification (BAI) algorithm \cite{audibert2010best}. The fixed-budget BAI algorithm  aims at finding the optimal arm with high probability in a fixed number of rounds \cite{audibert2010best, shahrampour2017sequential}. Consequently, the algorithm's performance does not concern the cumulative regret but the quality of the final recommendation, i.e., the probability that the recommended arm is optimal.\\
In each round, the \emph{general sequential elimination algorithm} \cite{shahrampour2017sequential} eliminates one \cite{audibert2010best,drugan2014scalarization} or more arms with the lowest average rewards, allowing the algorithm to scale well with the increasing number of arms (see Table 1 in \cite{shahrampour2017sequential}). Our proposed solution extends that algorithm to the multi-objective scenario with one crucial adjustment: the optimal arm is defined as the one that maximizes the ratio of the expected reward obtained from selecting an arm to the incurred cost. This cost depends on the network coverage and transmission delays and accounts for the random energy consumed by the UAVs in updating and transmitting the local models. Therefore, the proposed algorithm differs from the approach in \cite{xia2016best} that eliminates one arm per round assuming that the reward and cost are independent. Then, we find an upper bound on the error probability for our proposed algorithm. The fixed budget setting of the proposed algorithm allows us to estimate the required energy of the solution.
\subsection{Our Contributions}
{\color{black}
This paper considers a UAV-enhanced WSN in which the UAVs collect sensor data and engage in FL with a ground server. It addresses the problem of UAV placement, aiming to simultaneously maximize the network coverage and minimize the FL convergence delay under high uncertainty about the sensors' placement, activity, and channel conditions. As these two objectives are conflicting, this paper proposes solutions based on bi-objective-MABs (BO-MABs). The following list outlines the key contributions.
}
\begin{itemize}
\item Analyzing the components of the FL time and obtaining their statistical distributions.
\item Formulating the problem as an optimization problem with two conflicting objectives: maximizing the sensor coverage, and minimizing the total FL delay. 
\item Modelling and solving the problem in a BO-MAB framework using both the Pareto partial ordering and scalarization. Then, the simulation results for the two methods are presented.
\item Given the possibility of having many arms and the limited energy resources at the UAVs, the paper proposes a scalarization-based sequential elimination algorithm for BAI for the BO-MAB scheme. In this algorithm, the best arm is defined as the one that maximizes the ratio of the expected reward to the expected energy cost for an energy-efficient solution. Afterward, an upper bound to the error probability is derived for the proposed algorithm. 
\item The proposed algorithm is used to find the UAV configuration that achieves the best coverage and convergence delay per energy cost. The numerical results show the algorithm's performance for an FL network training the MNIST dataset for image classification. 
\end{itemize}
%
\subsection{Paper Organization}
\textbf{Section~\ref{sec:system_model}} presents the system model. \textbf{Section~\ref{sec:coverage_time}}, analyzes the FL convergence time and shows the dependency of the convergence time on the network coverage. In \textbf{Section~\ref{sec:problem_formulation}}, the problem is formulated as a bi-objective optimization problem. Then, in \textbf{Section~\ref{sec:modeling}}, the multi-objective MAB problem is described and then is used in  \textbf{Section~\ref{sec:solution_MOMAB}} to solve the joint problem of coverage maximization and convergence time minimization. In \textbf{Section~\ref{sec:sequential_elimination}}, a computationally efficient solution is proposed using a scalarization-based sequential elimination algorithm for best arm identification.  \textbf{Section~\ref{sec:numerical_results}} includes the numerical results and \textbf{Section~\ref{sec:conclusion}} concludes. 
\section{System Model} 
\label{sec:system_model}
The system model illustrated in \textbf{Fig.~\ref{fig:system_model}}. In this WSN, the sensors are located randomly in a square region of edge length $E$. The sensors' locations and density distribution are unknown to the UAVs and the ground server. All the sensors are homogeneous in their transmit power and data size, and each sensor relies on ambient energy harvesting to operate. The network also includes a set $\mathcal{U}$ of $U$ rotary-wing UAVs with identical computation capabilities. Each sensor collects data and transmits it to a UAV; thus, a sensor communicates with a UAV if (i) it is within that UAV's coverage region and (ii) it has sufficient energy. Due to the random nature of energy harvesting, the number of transmitting sensors (active) at a given time is random. Each UAV collects data from the active sensors in its coverage region and engages in the FL process with the ground server over a dedicated communication link. This system was motivated by many applications such as the IoT for agriculture and natural disaster analysis. {\color{black} The proposed framework is easily adaptable to other system models as long as network delay and coverage remain uncertain. The first arises due to the randomness in the communication delays and the UAVs' processing capabilities, whereas the latter arises by the randomness in the number of sensors or their availability.}
%
\begin{figure}[ht]
    \centering
    \includegraphics[scale=0.5]{ 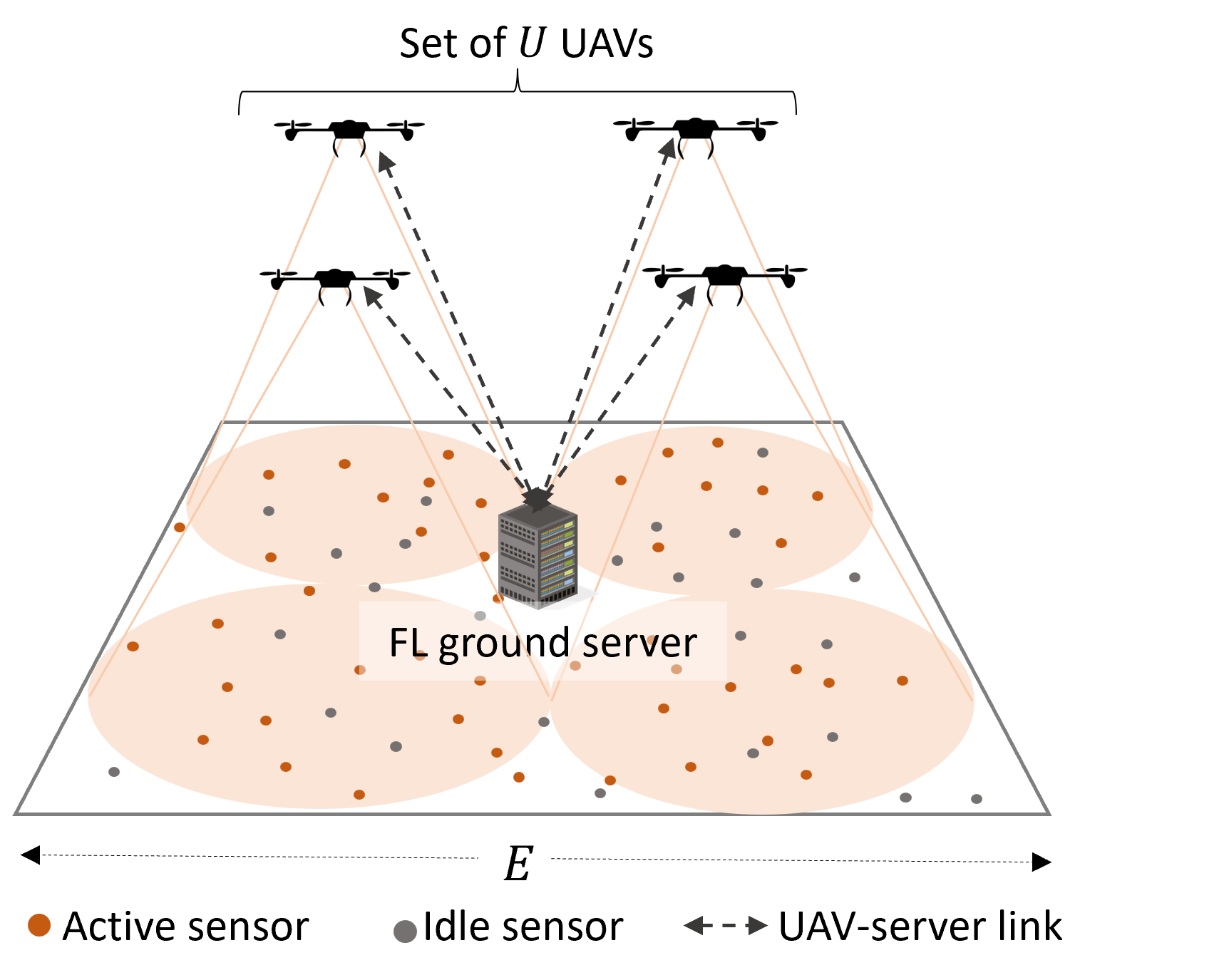}
    \caption{\small{A UAV-enhanced sensor network. Each UAV collects data from the active sensors in its coverage region and performs FL with the ground sever.}}
    \label{fig:system_model}
\end{figure}
\subsection{Federated Learning in the UAV-Enhanced Network}
Each UAV $i \in \mathcal{U}$ collects $D_i$ data samples $\{ \boldsymbol{x}_{il},y_{il}\}_{l=1}^{D_i}$ from the sensors in its coverage region, $\boldsymbol{x}_{il}\in \mathbb{R}^d$ is a $d$-feature input vector and $y_{il}\in \mathbb{R}$ is the corresponding label. Then, each UAV trains and updates the local model $\boldsymbol{\omega}$ using its $D_i$ data samples. We use $f_i\left(\boldsymbol{\omega},\boldsymbol{x}_{il},y_{il}\right)$ to denote the loss function for the $l$th data sample at UAV $i$. The average loss function of UAV $i$, $F_i(\boldsymbol{\omega},\boldsymbol{x}_{i1},y_{i1},\dots, \boldsymbol{x}_{iD_i},y_{i D_i})$, or simply $F_i(\boldsymbol{\omega})$, is
\begin{equation}
    F_i(\boldsymbol{\omega}) \coloneqq \frac{1}{D_i}\sum_{l=1}^{D_i}f_i(\boldsymbol{\omega}, \boldsymbol{x}_{il}, y_{il}).
\end{equation}
Let $D=\sum_{i=1}^{U} D_i$, the global loss function is defined as \cite{FL_Dinh2021}
\begin{equation}
    F(\boldsymbol{\omega}) \coloneqq \sum_{i=1}^U \frac{D_i}{D} F_i(\boldsymbol{\omega}) = \frac{1}{D}\sum_{i=1}^{U} \sum_{l=1}^{D_i} f_i(\boldsymbol{\omega}, \boldsymbol{x}_{il}, y_{il}),
\end{equation}
The goal of FL is to find a parameter  $\boldsymbol{\omega}^* \in \mathbb{R}^d$ that minimizes the global loss function
\begin{equation}
   \boldsymbol{\omega}^* = \underset{\boldsymbol{\omega \in \mathbb{R}^d}}{\arg \min} \, F(\boldsymbol{\omega}).
    \label{eq_FL}
\end{equation}
Without loss of generality, the \textit{FEDL} FL algorithm \cite{FL_Dinh2021} is used in this work to solve the problem in \eqref{eq_FL}. The algorithm assumes that $F_i(\boldsymbol{\omega})$ is $L$-smooth and $\beta$-strongly convex, which holds for several applications, such as the $l_2$-regularized linear regression model with $f_i(\boldsymbol{\omega}, \boldsymbol{x}_{il},y_{il})=\frac{1}{2}\left(\langle \boldsymbol{x}_{i,l}, \boldsymbol{\omega}\rangle-y_{i,l} \right)^2+\frac{\beta}{2} ||\boldsymbol{\omega}||^2$, where $\langle \boldsymbol{u},\boldsymbol{v} \rangle$ is the inner product of two vectors $\boldsymbol{u}$ and $\boldsymbol{v}$ and $||\boldsymbol{v}|| =\langle \boldsymbol{v},\boldsymbol{v} \rangle$ is the 2-norm of $\boldsymbol{v}$.\\
In brief, each global iteration $m$ of FEDL starts when the ground server transmits a model to the UAVs. Then, each UAV $i$ updates the model using the $D_i$ local data samples collected from the active sensors in its coverage region with a predetermined accuracy $\kappa \in (0,1)$. The local update time at UAV $i$ is
\begin{equation}
    T^\text{cp}_i =I(\kappa) \frac{\delta D_i}{f_{\text{CPU},i}}, 
\label{eq:FL_Tcp}
\end{equation}
where $I(\kappa)$ is the number of local iterations required to achieve local accuracy $\kappa$, $\delta$ (cycle/sample) is the number of computation cycles per data sample, and $f_{\text{CPU}}$ is the UAV's CPU frequency. Afterward, each UAV sends its current model to the server to update the global model for the next round. The global iterations continue until 
 the stopping condition is met. By definition, the FL converges when the 
difference between the current and the minimal loss functions is no more than the global accuracy $\epsilon$, i.e., $F(\boldsymbol{\omega}^{m})-F(\boldsymbol{\omega}^{*}) \le \epsilon$. However, since $F(\boldsymbol{\omega}^{*})$ is unknown to the server, the algorithm stops after a predetermined number of global iterations.
%
\begin{algorithm}
\small
\caption{\small{The FEDL Algorithm \cite{FL_Dinh2021}}} 
\label{Alg:FL}
\begin{algorithmic}[1]
\STATE In the first global iteration $m=0$, the server generates and broadcasts an initial global model $\boldsymbol{\omega}^{0}$ to all UAVs. It also determines the local accuracy parameter $\kappa\in (0,1)$.
\STATE In the following global iterations, each UAV $i$ receives $\boldsymbol{\omega}^{m-1}$ and $\nabla \bar{F}^{m-1}$ from the server. It then solves $\boldsymbol{\omega}_i^{m}\! =\! \underset{\boldsymbol{\omega} \in \mathbb{R}^d}{\text{arg}\min}\,{J_i^{m}( \boldsymbol{\omega}\!)}$,
{\small    
\begin{equation}
  {J_i^{m}(\boldsymbol{\omega})} \coloneqq  F_i(\boldsymbol{\omega})+ \langle\eta \nabla \bar{F}^{m-1}-\nabla F_i(\boldsymbol{\omega}^{m-1}),\boldsymbol{\omega}\rangle,
\end{equation}}
\noindent where  $\eta$ is the learning rate. The resulting $\boldsymbol{\omega}_i^{m}$ satisfies the accuracy $\kappa \in (0,1)$, i.e.,
\begin{equation}
   ||\nabla J_i^{m}(\boldsymbol{\omega}_i^{m})|| \le  \kappa ||\nabla J_i^{m}(\boldsymbol{\omega}^{m-1})||, \quad \forall i \in \mathcal{U},
\end{equation}
$\kappa=0$ solves the local problem optimally, whereas $\kappa=1$ means that the local model doesn't change.
\STATE The UAVs send their local models $\boldsymbol{\omega}_i^{m}$ and local gradients $\nabla F_i (\boldsymbol{\omega}_i^{m} )$ to the server. It collects them to calculate $\boldsymbol{\omega}^{m}=\sum_{i=1}^U \frac{D_i}{D} \boldsymbol{\omega}_i^{m}$ and $\nabla \bar{F}^{m} = \sum_{i=1}^U \frac{D_i}{D} \nabla \bar{F}_i(\omega_i^{m})$. The server then transmits these values to all UAVs.
\STATE The global iterations, i.e., the collection of updating and transmission processes repeats until convergence, i.e., until achieving a certain difference between the current and the minimal loss function. Formally, $F(\boldsymbol{\omega}^{m})-F(\boldsymbol{\omega}^{*}) \le \epsilon$.
\end{algorithmic}
\end{algorithm}
\subsection{Transmission Probability of Energy Harvesting Sensors}
\label{sec:energy_harvesting}
The wireless sensors obtain their required energy via ambient energy harvesting, {\color{black} a convenient energy resource in IoT in agriculture \cite{friha2021internet}}. The sensor use a harvest-use scheme. A sensor transmits its data to a UAV only if it lies in its coverage region and has sufficient transmission energy. And since ambient energy harvesting is non-deterministic, the sensors' activity is random \\
We model the energy arrival process at a sensor $m$  as a Poisson process with the rate $\lambda_m$. The amount of energy harvested in each arrival $v$, $Z_{m,v}$, follows an exponential distribution with parameter $\eta_{m,v}$, $Z_{m,v}{\sim}\text{EXP}(\eta_{m,v})$ \cite{maghsudi2017distributed}. The energy harvesting phase continues until the $k_m$ energy arrival. Afterward, the sensor transmits its data to its corresponding UAV. For known $\eta_{m,v}$, one can estimate the number of arrivals; otherwise, it can be selected randomly \cite{maghsudi2017distributed}. The total amount of energy at sensor $m$, denoted by $Y_m$, is the sum of $k_m$ energy arrivals, i.e., $Y_m = \sum_{v=1}^{k_m} Z_{m,v}.$\\
To simplify the analysis, we consider the case of independent and identically distributed (i.i.d) energy arrivals { \footnote{ The extension to the case of independent and non-identically distributed (i.ni.d) energy arrivals is straight forward as it only affects the probability of transmission. The distribution of $Y_m$ for this case is given in  \cite{maghsudi2017distributed} along with the conditions required for approximating the distribution of $Y_m$ by a Normal or an Erlang distributions.}}. Thus, $\eta_{m,v} = \eta_m$ for $v=\{1, \dots, k_m\}$ and the total energy is the sum of $k_m$ i.i.d exponential random variables. As a result, $Y_m$ follows the Erlang distribution with parameters $k_m$ and $\eta_m$, i.e., $Y_m {\sim} \text{Erl}\left( k_m, \eta_m \right)$
\begin{equation}
    f_{Y_m}(y_m) = \frac{{\eta_m}^{k_m}}{(k_m-1)!} y_m^{(k_m -1)} e^{- \eta_m y_m}.
\end{equation}
Let $E_{\text{T}}$ be the amount of energy required for a sensor $m$ to transmit to some UAV. Then, the probability that a covered sensor transmits to the corresponding UAV is
\begin{equation}
    \mathbb{P}\left[ Y_m \ge E_{\text{T}} \right] = \sum_{n=0}^{{k_m}-1} \frac{1}{n!} e^{-\eta_m E_{\text{T}}} (\eta_m E_{\text{T}})^n.
    \label{eq:p_sensor_transmit}
\end{equation}
\subsection{Network Coverage}
\label{subsec:network_coverage}
The network coverage can be maximized by finding the optimal locations of the UAVs in the 3D space and then relocating the UAVs accordingly. However, this approach is computationally costly and power inefficient. In addition, given no information about the sensor's deployment density, it is infeasible to find the optimal UAV placement that maximizes the sensor coverage. To overcome these problems, the UAVs are deployed according to the circle packing theorem in the 2D space and retain a fixed altitude $H$.\footnote{{\color{black}Our work applies to the case where the sensor's distribution is known, but their availability is random. In this case, other initial UAV placements might suit better than CS.}} Then, the coverage of the UAVs is controlled by adjusting the beamwidth of each UAV while maintaining its location. Details follow.\\ 
Each UAV $i$ has a directional antenna with an adjustable beamwidth. Similar to \cite{he2017joint}, the azimuth and elevation half-power beamwidths of the antennas are assumed to be the same and equal to $\theta_i \in (0,\frac{\pi}{2})$ in radians. The approximate antenna gain is \cite{balanis2015}
\begin{equation}
    G_i(\theta_i) = \begin{cases}
    G_0/\theta^2_i, & -\theta_i \le \varphi \le \theta_i \\
    g_0, & \text{else},
    \end{cases}
    \label{eq:gain}
\end{equation}
where $G_0 = 2.2846 $ \cite{he2017joint}, and $g_0$ is the gain in the side lobes. In practice, we have $0< g_0 \ll G_0/\theta_i^2 $. For simplicity, in what follows, we assume $g_0 = 0$ \cite{he2017joint, yang2019energy}.  In the rest of the paper, the subscript $\theta_i$ is added to the variables that depend on the beamwidth.\\
The radius of the coverage disk formed by the main lobe of the antenna of UAV $i$ is $R_{i,\theta_i} = H \tan(\theta_i)$.
The path loss of the link between UAV $i$ and a sensor located at the edge of the coverage region $R_{i,\theta_i}$ in dB yields \cite{alhourani2014}
\begin{equation}
  L_{i,\theta_i}^{\text{dB}} \! =\! \frac{A}{1+ a \exp\left(-b[\frac{180}{\pi} \tan^{-1} \left( \frac{H}{R_{i,\theta_i}} \right) -a]\right)}
  + 10 \log \left( R_{i,\theta_i}^2 +H^2\right) + Y,
    \label{eq:find_rmax}
\end{equation}
where $A=\eta_{\text{LoS}}-\eta_{\text{NLoS}}$ is the mean excessive loss in the line-of-sight (LoS) and non-LOS (NLoS) paths, while $a$ and $b$ are constants that depend on the environment \cite{alhourani2014}. Moreover, $Y=20 \log\left( \frac{4\pi f_c}{c}\right)+\eta_{\text{NLoS}}$, where $f_c$ is the carrier frequency and $c$ is the speed of light. \\
The maximum beamwidth $\theta_{i,\max}$ of UAV $i$ is determined by the minimum acceptable receive power $P_{\text{r,min}}$ at UAV $i$  defined to be
\begin{equation}
P_{\text{r},\min} \coloneqq P_{\text{t}} + G_{i,\theta_i} - L_{i, \theta_{i,\max}}^{dB},
\label{eq:find_theta_max}
\end{equation}
where $P_{\text{t}}$ is the power transmitted by a wireless sensor and it is equal for all sensors. $G_{i,\theta_i}$ is the antenna gain of UAV $i$ when its beamwidth is $\theta_i$. Given $\theta_{i,\max}$, one can determine the maximum coverage area $R_{i,\theta_{i,\max}}$ of UAV $i$. \\
The probability distribution of the number of sensors is unknown to the server. However, we assume this distribution belongs to the class of Poisson distributions parameterized by an unknown rate $\lambda$ (sensor/m$^2$). The probability distribution of the number of sensors located in the coverage disk of UAV $i$ of area $\Omega_{i,\theta_i} = \pi R_{i,\theta_i}^2$ is $\mathbb{P}[N_{i,\theta_i}^{\text{total}} = k]= \frac{(\lambda \Omega_{i,\theta_i})^k e^{-\lambda \Omega_{i,\theta_i}}}{k!}$. Now since \eqref{eq:p_sensor_transmit} gives the probability that a sensor has sufficient harvested energy and it is independent of $\mathbb{P}[N_{i,\theta_i}^{\text{total}} = k]$, the number of active sensors $N_{i,\theta_i}$ is  $\mathbb{P}[N_{i,\theta_i}= k] =\mathbb{P}[N_{i,\theta_i}^{\text{total}} = k] \mathbb{P} [Y_m \ge E_{\text{T}}]$. Hence, we have
\begin{equation}
        \mathbb{P}[N_{i,\theta_i} = k] = \frac{(\lambda \Omega_{i,\theta_i})^k e^{-\lambda \Omega_{i,\theta_i}}}{k!} \sum_{n=0}^{{k_m}-1} \frac{1}{n!} e^{-\eta_m E_{\text{T}}} (\eta_m E_{\text{T}})^n, \qquad k=0,1,\dots .
    \end{equation} 
The average number of active sensors covered by UAV $i$ is 
    \begin{equation}
        \bar{N}_{i,\theta_i} = \lambda \Omega_{i,\theta_i} \sum_{n=0}^{{k_m}-1} \frac{1}{n!} e^{-\eta_m E_{\text{T}}} (\eta_m E_{\text{T}})^n.
    \label{eq:avg_no_nodes}
    \end{equation}
\section{The FL Convergence Time and Its Dependency on Coverage}
\label{sec:coverage_time}
In the proposed UAV-enhanced network, the FL convergence time has three main components:
\begin{itemize}
    \item The local model update time at the UAVs (computation time);
    \item The data transmission time between the sensors and the UAVs ;
    \item The local- and global model exchange time between the UAVs and the server.
\end{itemize}
Changing a UAV’s beamwidth influences the number of its covered sensors, hence the number of its collected data samples and its model update time. In addition, the communication time over the sensor-UAV link depends on the beamwidth because it affects the distance between the sensor and the UAV and so the channel quality. More precisely, larger beamwidths worsen the path loss, reduce the antenna gain, and prolong the communication delays. In contrast, the communication delay between each UAV and the server depends on the length and quality of the UAV-server communication link.
%
\subsection{Computation Time}
\label{sec:computation_time}
The beamwidth of UAV $i$, $\theta_i$, determines the number of active sensors it covers and with it, the number of data samples $D_{i,\theta_i}$ it processes. Given that the sensors sample at the same frequency, i.e., there are $D$ samples per sensor, the average computation time of UAV $i$ yields
\begin{equation}
T^{\text{cp}}_{i,\theta_i}= I(\kappa) \frac{\delta D_{i,\theta_i}}{f_{\text{CPU}}}
=   I(\kappa) \frac{\delta D \bar{N}_{i,\theta_i} }{f_{\text{CPU}}},
\label{eq:Tcp_nodes}
\end{equation}
where $\bar{N}_{i,\theta_i}$ is the average number of active nodes given by \eqref{eq:avg_no_nodes}. That shows that a large beamwidth can result in a high computation time by increasing the coverage. {\color{black} To simplify the derivation of the delay probability distribution}, the UAVs' CPU frequency $f_{\text{CPU}}$ is assumed to be constant. The amount of energy consumed by UAV $i$ in updating its local model is
\begin{equation}
    E_{i,\theta_i}^{\text{cp}}=\frac{\alpha}{2}I(\kappa) \delta D_{i, \theta_i} f^2_{\text{CPU}},  
    \label{eq:energy_UAV}
\end{equation}
where $\alpha$ is the energy capacitance coefficient of the UAVs. \\
The following proposition gives the probability distribution function of the computation time at UAV $i$ for the Poisson nodes' distribution. The value of $T^\text{cp}_{i,\theta_i}$ in \eqref{eq:Tcp_nodes} is an integer multiple $\tau$ of $I(\kappa) \frac{\delta D}{f_{\text{CPU}}}$. The mean computation time at the UAV per unit of coverage area is $\nu = I(\kappa) \frac{\delta D \lambda}{f_{\text{CPU}}}$ second/m$^2$, where $\lambda$ is the sensor density. 
The probability distribution function of the computation time at UAV $i$, $T_{i,\theta_i}^{\text{cp}}$ is
\begin{equation}
    \mathbb{P}\left[T_{i,\theta_i}^{\text{cp}} = \tau \left(I(\kappa) \frac{\delta D}{f_{\text{CPU}}} \right)\right] =  \frac{(\nu \Omega_{i,\theta_i})^\tau e^{-(\nu \Omega_{i,\theta_i})}}{\tau!} 
    \sum_{n=0}^{{k_m}-1} \frac{1}{n!} e^{-\eta_m E_{\text{T}}} (\eta_m E_{\text{T}})^n, \quad \tau=0,1,2,\dots
\end{equation}
%
\subsection{The Communication Time over the UAV-Server Link} 
As long as a UAV does not move, its link to the ground server and the average delay of the server-UAV communication remains fixed. However, because of the random channel quality, the instantaneous delay can vary. The path loss between UAV $i$ and the ground server, denoted $ L_{i,\text{G}}^{\text{dB}}$ is obtainable by replacing $R_{i,\theta_i}$ in \eqref{eq:find_rmax} with $d_{i,\text{G}}$ which represents the horizontal distance  between the projection of UAV $i$ on the ground and the ground server. The data rate then yields
\begin{equation}
r_{i,\text{G}}= B_\text{G} \log \left[1 + \frac{P_{\text{t,G}} |g_{i,\text{G}}|^2 G_{i,\text{G}}}{L_{i,\text{G}} \, \sigma_\text{G}^2 } \right],
\label{eq:rate_GS}
\end{equation}
where $B_\text{G}$ is the bandwidth and $P_{\text{t,G}}$ is the transmission power of the ground server. Moreover, $g_{i,\text{G}}$ and $G_{i,\text{G}}$ respectively denote the gain of the Ricean channel and the antenna gain in the direction of the ground server. Finally, $\sigma_\text{G}^2$ is the noise variance. 
For the downlink and uplink, we assume equal transmission bandwidths and channel reciprocity. Therefore, the transmission time of the global- and local updates with size $S_\text{G}$ and $S_\text{L}$ bits, respectively, is given by 
\begin{equation}
T_{i,\text{G}}=\frac{S_{\text{G}}}{r_{i,\text{G}}}, \quad T_{i,\text{L}}=\frac{S_{\text{L}}}{r_{i,\text{G}}}.
\label{eq:T_global_model}
\end{equation}
The following proposition gives the statistical distribution of the communication time over the UAV-server link for UAV $i$. 
\begin{proposition}
\label{prop:T_uav_ground}
    $T_{i,\text{G}}$ is a random variable denoting the communication time of $S_M$ bits over the UAV-server link, its probability distribution is given by 
\begin{equation}
    f_{T_{i,\text{G}}}(t_{i,\text{G}}) = \frac{S_\text{M}}{2B_{\text{G}} \gamma_{i,\text{G}} \xi_{\text{G}}^2 t_{i,\text{G}}^2}
    \exp \left[ { \frac{k_{i,\text{G}}^2 \gamma_{i,\text{G}} + e^{\left(\frac{S_\text{M}}{B_{\text{G}} t_{i,\text{G}}} \right)} -1}{2\gamma_{i,\text{G}} \xi_{\text{G}}^2} }\right]
    I_0 \left( \frac{k_{i,\text{G}} \sqrt{e^\frac{S_\text{M}}{B_{\text{G}} t_{i,\text{G}}} -1}}{\xi_{\text{G}}^2 \sqrt{\gamma_{i,\text{G}}}}  \right),
\end{equation}
where $\gamma_{i,\text{G}} = \frac{P_{\text{t,G}} G_{i,\text{G}}}{L_{i,\text{G}} \sigma^2_{\text{G}}}$. $S_{\text{M}}$ stands for either the global update $S_{\text{G}}$ or the local update $S_{\text{L}}$. $k_{i,\text{G}}$ and $\xi_{\text{G}}$ represent the parameters of the $\chi^2$ distribution.
\end{proposition}
\begin{proof}
    See Appendix \ref{sec:appendix_T_uav_ground}
\end{proof}
%
\label{sec:communication_time}
\subsection{The Communication Time over the Sensor-UAV Link}
Similar to the computation time, the UAVs' beamwidth affects the communication time over the sensor-UAV link. By increasing the beamwidth, the antenna gain reduces, while the coverage area becomes bigger, implying a potentially longer UAV-sensor distance. Here, the focus is on the delay experienced by the farthest sensor from the projection of the UAV on the ground because it has the highest path loss. Due to the uncertainty in the sensors' locations, \textbf{Proposition~\ref{prop:max_distance}} obtains the expected value of the maximum distance between the UAV and the sensor using the order statistics. Later 
 this value is used in finding the average maximum delay. 
\begin{proposition}
The average maximum distance between UAV $i$ and a sensor located in its coverage region is
\begin{equation}
    \bar{Z}_{i,\theta_i} = \int_0^\infty z_i \lambda \bar{N}_{i,\theta_i} \Omega_{i,\theta_i} e^{\left(\lambda \Omega_{i,\theta_i} \right)z_i} \left[1 - e^{\left(\lambda \Omega_{i,\theta_i} \right)z_i}\right]^{\bar{N}_{i,\theta_i}-1} dz_i
    \end{equation}
    \label{prop:max_distance}
\end{proposition}
\begin{proof}
    See Appendix \ref{sec:appendix_max_distance}
\end{proof}
The farthest sensor $F$ from UAV $i$ within its coverage disk experiences the highest path loss denoted by $L_{\bar{Z}_{i,\theta_i}}^{\text{dB}}$, its value is obtainable from \eqref{eq:find_rmax} by replacing $R_{i,\theta_i}$ with $\bar{Z}_{i,\theta_i}$. The sensor's transmission rate, $r_{\bar{Z}_{i,\theta_i}}$, yields
\begin{equation}
r_{\bar{Z}_{i,\theta_i}} = B_{F} \log_2 \left[ 1 + \frac{P_{\text{t}}|g_F|^2 G_{i,\theta_i} } 
{ L_{{\bar{Z}_{i,\theta_i}}}  \sigma_{\text{F}}^2 }
\right],
\label{eq:rate_sensor}
\end{equation}
where $B_{\text{F}}$ is the channel bandwidth and $P_{\text{t}}$ is the sensor's transmit power. In addition, $g_{\text{F}}$ denotes the gain of the Ricean channel at the farthest sensor and $G_{i,\theta_i}$ is the antenna gain of UAV $i$. $\sigma_{\text{F}}^{2}$ is the noise variance. {\color{black} Here, the communication between the sensors and the UAVs is over orthogonal channels that mitigate interference, this work can be extended to other scenarios with interfering sensors by replacing \eqref{eq:rate_sensor} with the expression $r_{\bar{Z}_{i,\theta_i}} = B_{F} \log_2 \left[ 1 + \frac{P_{\text{t}}|g_F|^2 G_{i,\theta_i} L_{{\bar{Z}_{i,\theta_i}}}^{-1}} 
{ I_{\Psi}  +  \sigma_{\text{F}}^2 }  \right]$,
where ${\Psi}$ is the set of active sensors interfering with sensor $F$, and $ I_{\Psi}$ is the interference power. Interference reduces the data rate and results in longer transmission delays. }
The transmission time of $S_{\text{data}}$ bits then yields
\begin{equation}   T_{\bar{Z}_{i,\theta_i}}^{\text{cm}}=\frac{S_{\text{data}}}{r_{\bar{Z}_{i,\theta_i}}}.
\label{eq:T_sensor}
\end{equation}
\begin{proposition}
The probability distribution of the maximum delay over the sensor-UAV link assuming no overlap between the UAVs' coverage regions is
    \begin{equation}
        f_{T_{\bar{Z}_{i,\theta_i}}}(t_{\bar{Z}_{i,\theta_i}}) = \frac{S_\text{data}}{2B_\text{F}\alpha_{\bar{Z}_{i,\theta_i}} \xi_F^2 t_{\bar{Z}_{i,\theta_i}}^2}
        \exp \left[  \frac{k_F^2 \alpha_{\bar{Z}_{i,\theta_i}} + e^\frac{S_\text{data}}{B_\text{F} t_{\bar{Z}_{i,\theta_i}}} -1}{2\alpha_{\bar{Z}_{i,\theta_i}} \xi_F^2} \right]
        I_0 \left( \frac{k_{\text{F}} \sqrt{e^\frac{S_\text{data}}{B_F t_{\bar{Z}_{i,\theta_i}}} -1}}{\xi_\text{F}^2 \sqrt{\alpha_{\bar{Z}_{i,\theta_i}}}}  \right)
    \label{eq:T_sensor_uav}
    \end{equation}
where $\alpha_{\bar{Z}_{i,\theta_i}} =\frac{P_{\text{t}}|g_F|^2 G_{i,\theta_i} } 
{L_{{\bar{Z}_{i,\theta_i}}} \sigma_{\text{F}}^2 }$. $\xi_\text{F}$ and $k_\text{F}$ are the parameters of the $\chi^2$ distribution.
\end{proposition}
\begin{proof}
The proof of this proposition is similar to the proof of the distribution of the transmission time over the UAV-server link given by \textbf{Proposition~\ref{prop:T_uav_ground}}  because the channel has a Ricean distribution in both cases. 
\end{proof}
\subsection{The Total FL Convergence Time}
The overall communication time is 
\begin{equation}
T_{i,\theta_i}^{\text{cm}} = T_{{\bar{Z}_{i,\theta_i}}} + T_{i, \text{G}} + T_{i, \text{L}}.
\label{eq:total_communication_time}
\end{equation}
To find the distribution of the communication time, the simplifying assumption of equal sized local and global updates is made to have equal transmission times. A generalization to the uneven case is straightforward. Since the transmission time over the sensor-UAV link $t_{{\bar{Z}_{i,\theta_i}}} >0 $ is independent from the transmission time over the UAV-server link, $t_{i,\text{G}} >0$ the probability distribution of the total communication time $t_{{\bar{Z}_{i,\theta_i}}}$ is given by their convolution \cite{papoulis02}\\
\begin{equation}
    f_{T^{\text{cm}}_{i,\theta_i}} (t^{\text{cm}}_{i,\theta_i}) =
       \int_0^{t^{\text{cm}}_{i,\theta_i}}          f_{T_{\bar{Z}_{i,\theta_i}}}(t_{{\bar{Z}_{i,\theta_i}}})
        f_{T_{i,\text{G}}}(t^{\text{cm}}_{i,\theta_i} - t_{{\bar{Z}_{i,\theta_i}}}) d t_{{\bar{Z}_{i,\theta_i}}}.
\end{equation}
The total convergence time is
\begin{equation}
T_{i,\theta_i}^{\text{total}} = T_{i,\theta_i}^{\text{cp}} + T_{i,\theta_i}^{\text{cm}}.
\end{equation}
Finally, \textbf{Proposition~\ref{prop:t_total}} gives the cumulative distribution function of the total delay. 
\begin{proposition}
The cumulative distribution function of the total delay is
  \begin{equation}
      F_{T_{i,\theta_i}^{\text{total}}} \left(t_{i,\theta_i}^{\text{total}} \right) = 
      \sum_{\tau \in \bar{N}_{i,\theta_i}} \mathbb{P} \left[ t_{i,\theta_i}^{\text{cm}}  \le t_{i,\theta_i}^{\text{total}} - \tau |T_{i,\theta_i}^{\text{cp}}  = \tau   \right]
      \mathbb{P}\left[T_{i,\theta_i}^{\text{cp}} = \tau \left(I(\kappa) \frac{CD}{f_{\text{CPU}}}\right) \right]
  \end{equation}
  \label{prop:t_total}
\end{proposition}
\begin{proof}
    See Appendix~\ref{sec:appendix_t_total}
\end{proof}
\section{Problem Formulation}
\label{sec:problem_formulation}
The server aims at selecting suitable beamwidths for all UAVs that achieve the following conflicting objectives: (i) maximizing the network coverage to incorporate the data of more sensors in FL and (ii) minimizing the time duration of a global iteration. 
\subsection{Objective 1: Maximizing Coverage}
One objective is to maximize the number of sensors uniquely covered by each UAV. That means maximizing the number of active sensors covered by the UAVs while minimizing the overlap between the UAVs’ coverage regions. The problem is challenging as the sensors' distribution is unknown, and the number of active sensors is a random variable due to the random nature of energy harvesting. Formally, 
\begin{equation}
\underset{\boldsymbol{\theta}}{\text{maximize~~}}\sum_{i=1}^U  N_i(\theta_i)- N_{\text{OL}}(\boldsymbol{\theta}),
\label{eq:NoSensors}
\end{equation}
where $\boldsymbol{\theta} = \{\theta_1, \theta_2, \dots, \theta_U\}$ is the set of beamwidths assigned to the UAVs by the server. $N_i(\theta_i)$ is the number of active sensors covered by UAV $i$ when it has beamwidth $\theta_i$ including the overlapping active sensors that are simultaneously covered by other UAVs. $N_{\text{OL}}$ is the number of times active sensors are covered simultaneously by more than one UAV.
\subsection{Objective 2: Minimizing the Maximum Delay Time}
The second objective is to minimize the duration time of one global iteration, as described in \textbf{Section \ref{sec:coverage_time}}. In synchronized FL, the server updates the global model only after receiving the local models of all UAVs. Therefore, the aim is to minimize the maximum update time in a global iteration. That is,
\begin{equation}
\underset{\boldsymbol{\theta}}{\text{minimize} } \, \left[\underset{i \in \mathcal{U}}{\max} \, T_{i,\theta_{i}}^{\text{total} }\right].
\end{equation}
%
\section{Bi-Objective MAB Model}
\label{sec:modeling}
In realistic scenarios, there is often no a priori information about the sensors' distribution, activity, or the channel quality. Therefore, we propose solving the problem using a bi-objective MAB model.\\ 
In proposed model, the decision-maker (bandit agent) is the server, and the action set is the set of UAVs' beamwidth values. Thus, at every round, the server selects one configuration (arm) for the UAVs from a total of $K$ possible configurations, $\boldsymbol{\theta}_{k} = [\theta_{1,k}, \theta_{2,k}, \dots, \theta_{U,k} ]$ where $\theta_{i,k}$ is the beamwidth of UAV $i$ when the configuration $k$ is selected. Consequently, the server receives a reward vector $\boldsymbol{X}(\boldsymbol{\theta}_k) = [X^{\text{cov}} (\boldsymbol{\theta}_{k}), X^{\text{delay}}(\boldsymbol{\theta}_{k}) ]$ with two elements, one for each objective. The first element of the reward vector concerns coverage so that
\begin{equation}
    X^{\text{cov}}(\boldsymbol{\theta}_{k}) = \left( \sum_{i=1}^U N_i(\theta_{i,k}) -  N_{\text{OL}} (\boldsymbol{\theta}_k)  \right)_{\text{normalized}}.
\label{eq:reward_d1}
\end{equation}
The second element is related to the total delay, that is,
\begin{equation}
     X^{\text{delay}}(\boldsymbol{\theta}_k) = \left(  \frac{1}{\underset{i\in \mathcal{U}}{\max} \, T_{i,\theta_{i,k}}^{\text{total}}} \right)_{\text{normalized}}.
     \label{eq:reward_d2}
\end{equation}
Let $\boldsymbol{\mu}_k = [\mu_k^1, \mu_k^2]$ be the expected reward vector of arm $k$. The BO-MAB algorithm must find the optimal arms that maximize the expected reward vector $\boldsymbol{\mu}_k$.\\
The possible values for $\theta_i \in [\theta_{\min}, \theta_{\max}]$ are bounded from above, $\theta_{\max}$, by the minimum received power at the UAV, as given by \eqref{eq:find_theta_max}. Larger upper bounds result in larger coverage areas and a lower received power. At the same time, the difference between the upper- and lower limit determines the level of heterogeneity in the network, which affects the FL performance. In practical systems, $\theta_i$ takes values on a discrete set of values. 
Let $W$ be the number of beamwidths that the server can select for one UAV; then the total number of arms is
\begin{equation}
K=W^U.
\label{eq:exp_arms}
\end{equation}
For example, in a network with three UAVs, each with two possible angles, the server has $2^3=8$ possible configuration choices. 
\section{Solution for the Bi-Objective MAB Problem}
\label{sec:solution_MOMAB}
A BO-MAB problem has multiple optimal solutions according to the specified reward vector. The set of reward vectors that are non-dominated by any other reward vectors is the \textit{Pareto optimal reward set} $\mathcal{O}^{*}$. The arms whose reward vectors belong to $\mathcal{O}^{*}$ form the \textit{Pareto optimal arm set} $\mathcal{A}^{*}$ \cite{drugan2013}. We use two methods to find optimal solutions for the beamwidth selection problem in Section \ref{sec:modeling}: the Pareto and the scalarized upper-confidence-bound 1 (UCB1).
\subsection{Pareto UCB1}
\label{subsec:pareto_ucb1}
This algorithm directly maximizes the reward vectors in the
two-objective reward space. First, the agent pulls each arm once for initialization. Then, at each iteration, it finds the Pareto optimal arm set $\mathcal{A^{'}}$ that corresponds to the optimal rewards. For all non-optimal arms $ \alpha \not\in \mathcal{A^{'}}$, there exists a Pareto optimal arm $k \in \mathcal{A^{'}}$ that dominates arm $\alpha$, i.e.,
\begin{equation}
\bar{\boldsymbol{X}}_{\alpha} + \sqrt{ \frac{2 \ln (n \sqrt[4]{Q |\mathcal{A^{*}}|} )}{n_{\alpha}}} \not\succ \bar{\boldsymbol{X}}_k + \sqrt{ \frac{2 \ln (n \sqrt[4]{Q |\mathcal{A^{*}}|} )}{n_k}}
\label{eq:UCB}
\end{equation}
where $\bar{\boldsymbol{X}}_\alpha$ is the empirical mean reward of arm $\alpha$, the term $\sqrt{ \frac{2 \ln (n \sqrt[4]{Q |\mathcal{A^{*}}|} )}{n_{\alpha}}}$ is the confidence interval for arm $\alpha$, $n_\alpha$ is the number of pulls of arm $\alpha$, and $Q$ is the number of objectives, $Q=2$ here. Since the number of Pareto optimal arms is not known to the player a priori, the term $\sqrt[4]{Q |\mathcal{A^{*}}|}$ in \eqref{eq:UCB} can be replaced with $\sqrt[4]{Q K}$ \cite{drugan2013}. The notation $\boldsymbol{u} \not\succ \boldsymbol{v}$ indicates that the vector $\boldsymbol{v}$ is non-dominated by $\boldsymbol{u}$, meaning that there exists at least one dimension $q$ for which $v^q > u^q$. After finding the set of optimal arms $\mathcal{A}^{'}$, the algorithm pulls an arm $k \in \mathcal{A^{'}}$ uniformly at random and observes its reward. It then updates its average reward vector $\bar{\boldsymbol{X}}_\alpha$ and the counters. The process continues until meeting the stopping condition by finding the set of optimal arms.

The performance of the algorithm is measured by the Pareto regret; a measure of the distance between a sub-optimal reward vector $\boldsymbol{\mu}_\alpha$ and the Pareto optimal rewards. The Pareto regret of a sub-optimal arm $\alpha$ is the minimum positive value $\epsilon_\alpha$, denoted $\epsilon^{*}_{\alpha}$, that, when added to both objectives, results in a virtual optimal reward vector $\boldsymbol{\nu}_\alpha^{*}$ that belongs to the Pareto optimal reward set, i.e,  $\boldsymbol{\nu}_\alpha^{*} = \boldsymbol{\mu}_\alpha + \boldsymbol{\epsilon}^{*}_\alpha \in \mathcal{O}^{*}$ \cite{drugan2013}. {\color{black} The computational complexity for each round of the Pareto UCB1 algorithm is $\mathcal{O}(K \log(K))$ \cite{deb2011multi}, this is to find the Pareto set in a bi-objective problem. The overall complexity of the algorithm is $\mathcal{O}(n K \log(K))$.}
\begin{algorithm}
\small
\caption{\small{Pareto UCB1 \cite{drugan2013}}} 
\label{Alg:ParetoUCB1}
\begin{algorithmic}[1]
\STATE Play each arm once, $n \leftarrow K$, $n_k \leftarrow 1, \forall k$
\REPEAT 
    \STATE Find the Pareto set $\mathcal{A^{'}}$ such that $\forall k \in \mathcal{A^{'}}$ and $\forall \alpha $, \eqref{eq:UCB} holds.
    \STATE Pull arm $k$ chosen uniform randomly from the set $\mathcal{A^{'}}.$
    \STATE $n \leftarrow n+1$, $n_k \leftarrow n_k +1.$
    \STATE Update $\bar{\boldsymbol{X}}_k$.
\UNTIL stopping condition is met.
\end{algorithmic}
\end{algorithm}
%
\subsection{Scalarized Bi-Objective UCB1}
\label{subsec:scalarized_ucb1}
The scalarized bi-objective UCB1 method transfers a bi-objective reward into a single-objective using a scalarization function. To find the Pareto front, a set of $S$ scalarization functions to obtain the different elements in the Pareto front set. A widely-considered choice is the linear scalarization (Lin-S) function, {\color{black}
\begin{equation}
f^j(\boldsymbol{\mu}_k)=\omega_j^{1} \mu_k^1+ \omega_j^{2} \mu_k^2, \quad \forall k \in \mathcal{A}, j \in S,
\end{equation}
where $\omega_j^1, \omega_j^2$ are the weights assigned to the first and second objective under scalarization $j$, such that  $\omega_j^1 + \omega_j^2 =1$. The linear function is simple to implement but is incapable of finding all the optimal arms in a non-convex Pareto set \cite{drugan2013, miettinen1999nonlinear}. Alternatively, the Chebychev scalarization (Cheb-S) function can, under some conditions, find all the optimal arms in a non-convex Pareto set \cite{miettinen1999nonlinear,drugan2013}. The Chebychev scalarization function is given by
\begin{equation}
f^j(\boldsymbol{\mu}_k)= \underset{q \in \{1,2\}}{\min} \quad \omega_j^{q}( \mu_k^q -z^q) \quad \forall k \in \mathcal{A}, j \in S,
\end{equation}
the reference point $z^q = \underset{1 \le k \le K}{\min} \quad \mu_k^q - \varepsilon^q$, where $\varepsilon^q$ is a small positive number, $\varepsilon^q >0$.  \\
}
For a given scalarization function $f^j$, the optimum reward $\boldsymbol{\mu}^{*}$ maximizes $f^j(\boldsymbol{\mu}_k)$, i.e.,
\begin{equation}
    f^j(\boldsymbol{\mu}^{*}) \coloneqq \underset{1 \le k \le K}{\max} f^j(\boldsymbol{\mu}_k).
\end{equation}
\textbf{Algorithm \ref{Alg:Scalarized}} summarizes the scalarized bi-objective UCB1. For each of the scalarization functions $f^j$, the agent plays each arm once for initialization. It then updates the counters $n^j$ for the number of times $f^j$ is played, and $n_k^j$ for the number of times arm $k$ is played under scalarization $j$. Afterward, it selects a function $f^j$ uniformly at random, and plays the arm $k^{*}$ that maximizes the UCB term $f^j(\bar{\boldsymbol{X}}_k) + \sqrt{2 \ln  (n^j)/ n_k^j}$. Next, it updates the average reward of arm $k$, $\bar{\boldsymbol{X}}_k$ and the counters. \\
The regret for scalarization function $f^j$ and sub-optimal arm $\alpha$ is the \textit{scalarized regret},
\begin{equation}
\Delta_{\alpha}^j \coloneqq \underset{1\le k\le K}{\max} f^j(\boldsymbol{\mu}_k) - f^j(\boldsymbol{\mu}_{\alpha}).
\end{equation}
However, this metric does not show the reduction in the bi-objective regret because it combines a set of independent regrets. Thus, we use the unfairness regret to incorporate the mean rewards of all optimal arms \cite{drugan2013}
\begin{equation}
\mathfrak{R}_f(n) = \frac{1}{|\mathcal{A}^{*}|}  \sum_{k\in \mathcal{A}^{*}} (T_k^{*}(n) - \mathbb{E}[T^{*}(n)])^2,
\end{equation}
where $T_k^{*}(n)$ is the number of pulls of arm $k$ and $\mathbb{E}[T^{*}(n)]$ is the expected number of times the optimal arms are pulled. {\color{black} The computational complexity for each scalarization function and each round is $\mathcal{O}(K)$.  Let $n$ be the number of times a scalarization function $f^j \in \{f^1, \dots, f^S \}$ is sampled, then the total complexity is $\mathcal{O}(KSn)$.}
\begin{algorithm}
\small
\caption{\small{Scalarized Bi-Objective UCB1 \cite{drugan2013}}} 
\label{Alg:Scalarized}
\begin{algorithmic}[1]
\STATE \textbf{Input}: $\mathcal{S}=\{ f^1, f^2, \dots, f^S\}$
\FOR{$j =1:S$}
    \STATE Initialize the arms.
    \STATE Play each arm once and observe the reward $\boldsymbol{X}_k^j$.
    \STATE Update $n^j \leftarrow n^j+1, n_k^j \leftarrow n_k^j+1, \bar{\boldsymbol{X}}_k^j $.
\ENDFOR
\REPEAT 
    \STATE Choose a function $f^j$ from $\mathcal{S}$ uniformly at random.
    \STATE Pull arm $k^{*}$ that maximizes $f^j(\bar{\boldsymbol{X}}_k) + \sqrt{2 \ln (n^j)/ n_k^j}$.
    \STATE Update $\bar{\boldsymbol{X}}_{k^{*}}$, and set $n_{k^{*}}^j \leftarrow n_{k^{*}}^j +1, n^j \leftarrow n^j + 1$.
\UNTIL stopping condition is met.
\end{algorithmic}
\end{algorithm}\\
{\color{black}\begin{remark} The bandit algorithms used in this work are extendable to dynamic environments with non-stationary statistical distributions \cite{10044185}.\end{remark}}
\section{Sequential Elimination For Scalarization-Based Best Arm Identification}
\label{sec:sequential_elimination}
\textbf{Section~\ref{sec:solution_MOMAB}} discussed two methods for solving the bi-objective optimization problem under uncertainty.  However, those methods can be inefficient for a large number of UAVs or beamwidths due to the exponential growth in the number of arms, as given in \eqref{eq:exp_arms}. Therefore, this section proposes a novel solution using a generalized BAI algorithm that can control the number of arms eliminated at each round, allowing it to scale well for several arms. To ensure the energy efficiency of the proposed solution, it aims at finding the arm that maximizes the ratio of its expected reward to its expected energy cost. Furthermore, as this algorithm separates the exploration phase from the exploitation phase (pure exploration), the amount of UAV energy consumed during exploration is controlled based on the UAVs' energy budget. \\
The proposed algorithm extends the general sequential elimination algorithm for best arm identification in \cite{shahrampour2017sequential} to (i) include the scalarization-based bi-objective MAB algorithm, and (ii) find the arm that maximizes the ratio between the expected reward to the expected cost. The algorithm studies a fixed budget setting, where the agent aims to find the optimal arm with a high probability in a pre-determined number of rounds (budget). It is significant as it generalizes the successive rejects (Succ-Rej) algorithm for identifying the best arm \cite{audibert2010best} by eliminating one arm in each round of the algorithm. It is also a generalization of the sequential halving (Seq-Halv) algorithm that removes half of the remaining arms in each round \cite{shahrampour2017sequential}. Besides, \textit{Shahrampour et. al.} develop a new algorithm, namely, \textit{non-linear sequential elimination} (N-Seq-EL), which extends the Succ-Rej algorithm. Both algorithms eliminate one arm in every round, but the latter divides the remaining budget by a non-linear function of the remaining arms, a trick that can improve the best arm identification probability.\\
With each arm selection, the server receives $\boldsymbol{X}_k = [X^{\text{cov}}(\boldsymbol{\theta}_k), X^{\text{delay}}(\boldsymbol{\theta}_k) ]$ a two-dimensional reward vector as defined in \eqref{eq:reward_d1} and \eqref{eq:reward_d2}, with expected value $\boldsymbol{\mu}^X_k$. Also, this arm selection results in random energy cost $c_k$, with mean value $\mu^c_k$, affecting both objectives equally. This cost represents the UAV's consumed energy in updating their local models and transmitting their local updates to the server. It is given by
\begin{equation}
    c_k = \sum_{i=1}^{U} E_{i,\theta_i}^{\text{cp}} +E_i^{\text{cm}},
\end{equation}
where $E_{i,\theta_i}^{\text{cp}}$ is given in \eqref{eq:energy_UAV}, $E_i^{\text{cm}}= P_{\text{t},i} T_{i,\text{L}}$, and $ P_{\text{t},i}$ is the transmit power of UAV $i$. \\
The server aims to find a beamwidth configuration (arm) $k \in \mathcal{K}$ that maximizes the ratio of the expected reward to the expected energy cost,
\begin{equation}
   k^{*} = \underset{k \in \{1,\dots,K\} } {\arg \max} \frac{f^j(\boldsymbol{\mu}^X_k)}{\mu^c_k}.
\end{equation} 
\textbf{Algorithm~\ref{Alg:sSR}} summarizes the scalarized cost-aware general sequential elimination (S-C-G-Seq-El) method. It starts by estimating the budget $n^{'}$, i.e., the number of rounds for all scalarization functions, based on the UAV's energy budget. Let $E_{\text{total}}$ be the energy available at a UAV for learning and transmission, i.e., exploration. The fixed budget for running the algorithm is upper bounded by 
\begin{equation}
    n^{'} =\frac{E_{\text{total}}}{|\mathcal{S}| \left( \frac{\alpha}{2} \lambda \Omega_{\theta_{\max}}D  I(\kappa) \delta f^2_{\text{CPU}} +P_{\text{t}} T_{\text{L},\max}\right) },
    \label{eq:alg_budget}
\end{equation}
where $\Omega_{\theta_{\max}}$ is the maximum coverage area of a UAV, and $T_{\text{L},\max}$ is the maximum allowed upload time. The denominator in \eqref{eq:alg_budget} represents the maximum energy consumed by a UAV. It occurs when all the sensors in its coverage region are active and the upload delay is maximized.\\
After estimating the budget, the S-C-G-Seq-El algorithm proceeds to find the optimal arms $\boldsymbol{I}^{*}_S$ from the Pareto front by identifying the optimal arm for each scalarization function in the set $\mathcal{S}= \{f^1, \dots, f^j, \dots, f^{|S|}\}$. Initially, the set of optimal arms $\boldsymbol{I}^{*}_S$ is empty. Besides, the set of active arms for each scalarization function $f^j$ is the set of all arms $\mathcal{K}$, formally, $\boldsymbol{G}_1^j \leftarrow \mathcal{K}$. Then for each scalarization function, the algorithm runs for $R$ rounds, $R = K$ for successive rejects and $R= \lceil \log_2  K \rceil $ in sequential halving. In each round $\rho$, it pulls each arm $k$ in the set of active arms $\boldsymbol{G}_\rho^j$ for $n_\rho - n_{\rho-1}$ times,
\begin{equation}
    n_\rho = \left\lceil \frac{n -K}{C z_\rho} \right\rceil, \quad  \rho=1,\dots, R,
\end{equation}
where $n=\lfloor n^{'}/|S| \rfloor$ is the budget for each scalarization function, with $n^{'}$ being the total budget. Besides, $\{z_\rho\}_{\rho=1}^{R}$ is a decreasing positive sequence, and $C$ is a constant $C=\frac{1}{z_R}+\sum_{\rho=1}^{R}\frac{b_\rho}{z_\rho}$ \cite{shahrampour2017sequential}.\\
After pulling the arms, the algorithm discards the subset $\boldsymbol{B}_\rho^j$ of $b_\rho$ arms with the lowest ratio of expected reward to the expected cost. It also updates the set of active arms, $\boldsymbol{G}_{\rho+1}^j = \boldsymbol{G}_{\rho}^j \backslash\boldsymbol{B}_{\rho}^j$. The process continues for $R$ rounds to eliminate $K-1$ arms. The remaining arm $k^{*}$ is the optimal one for the scalarization function $f^j$, so the algorithm includes it in the set of optimal arms $\boldsymbol{I}^{*}_S$. The cycle repeats for all the scalarization functions to identify all optimal arms. {\color{black} The computational complexity of the S-C-G-Seq-El algorithm is $\mathcal{O}(SRn)$.
}
\begin{algorithm}[ht]
\caption{\small{The Scalarized Cost-Aware General Sequential Elimination Algorithm (S-C-G-Seq-El)}}
\begin{algorithmic}[1]
    \STATE \textbf{Input:} Set of arms $\mathcal{K}$, set of scalarization functions $\mathcal{S}$, budget $n^{'}$, sequence $\{z_\rho, b_\rho\}_{\rho=1}^R$.
    \STATE \textbf{Initialization:} $\boldsymbol{I}^{*}_S \leftarrow \phi, \boldsymbol{G}_1^j \leftarrow \mathcal{K},\, \forall  f^j \in S $.
    \STATE Let $n_0 =0$,
        \\ $C = \frac{1}{z_R} + \sum_{\rho=1}^{R} \frac{b_\rho}{z_\rho}$.\\
        $n_\rho = \left\lceil \frac{n -K}{C z_\rho} \right\rceil$ for $\rho = 1, \dots,R$.
    \FORALL{$f^j \in \mathcal{S}$} 
        \FORALL{rounds $\rho = 1, \dots, R$} 
           \STATE Select each arm $k \in \boldsymbol{G}_\rho^j$ for $n_\rho - n_{\rho-1}$ times.
           \STATE Find the set of $b_\rho$ arms with the smallest ratio of the expected reward to the expected cost, i.e., $\underset{k \in \boldsymbol{G}_\rho^j}{\arg \min} \, f^j(\bar{\boldsymbol{X}}_{k,n_\rho})/\bar{c}_{k,n_\rho}$, denoted by $\boldsymbol{B}_\rho^j$.
           \STATE Discard the worst $b_\rho$ arms. Let $\boldsymbol{G}_{\rho+1}^j = \boldsymbol{G}_{\rho}^j \backslash\boldsymbol{B}_{\rho}^j$.
        \ENDFOR
        \IF{$k^{*} \not \in \boldsymbol{I}^{*}_S$ where $k^{*} = \boldsymbol{G}_{R+1}^j$}
            \STATE $\boldsymbol{I}^{*}_S \leftarrow \boldsymbol{I}^{*}_S \cup \{ k^{*}\}$.
        \ENDIF
    \ENDFOR
\end{algorithmic}
\label{Alg:sSR}
\end{algorithm}
\begin{theorem}
The probability of misidentifying a sub-optimal arm as the optimal one for a scalarization function $f^j$ is
\begin{equation}
      \mathbb{P} \left[ \textbf{G}_{R+1}^j \neq \{ k^{*} \}\right] \le 
       R \underset{\rho \in \{1,\dots,R\}}{\max}\{b_\rho\} \underset{\rho \in \{1,\dots,R\}}{\max} \left\{ 2 \exp\left(- 2 n_\rho \Xi_\rho^2 \right)\right\}.
\end{equation}
\label{theorem:error_propability}
\end{theorem}
\begin{proof}
  See Appendix~\ref{sec:appendix_theorem}
\end{proof}
\section{Numerical Results}
\label{sec:numerical_results}
\begin{table}[]
{\color{black}
    \caption{Simulation Parameters}
    \centering
    \begin{tabular}{ccc||ccc}
        Parameter  & Value & Description & Parameter  & Value & Description\\  \hline 
        $\lambda$  & 150 $\frac{\text{sensor}}{\text{km}^2}$ & Sensor density
        &  $P_t$  & -3 dBW & Transmit power for all sensors and UAVs   \\ \hline
        $H$  & 230 m & UAV atitude  
        &  $P_{\min}$  & $-100$ dB & Minimum receive power at a UAV \\ \hline %
        $B_{\text{G}}, B_{\text{F}}$  & 1 MHz & Channel bandwidth         
        &    $b$  & 0.16 & Urban environment S-curve parameter  \\ \hline
        $\eta_{\text{LoS}}$  & 1 dB & Mean additional LoS losses          
        & $a$  & 9.61 & Urban environment S-curve parameter       \\ \hline
         $\eta_{\text{NLoS}}$  & 20 dB & Mean additional NLoS losses &$f_c$  & 2.5 GHz & Transmission frequency         
               \\ \hline
        $\alpha$ & $2 \times 10^{-28}$ & Capacitance coefficient & $E_{\text{total}}$ & $20$ J & UAV's learning and transmission energy\\
        \hline
    \end{tabular}
    \label{tab:simulation_parameters}
    }
\end{table}
We consider a $1$ km $\times$ $1$ km square-shaped sensor network. The sensors' distribution is a Poisson point process with density $\lambda=150$ sensor/km$^2$. The ground server is at the center. The locations of the four available UAVs are determined by the circle packing theorem. The UAVs' altitude is $H=230$ m, which is optimal for the circle packing radius $R=250$ m \cite{alhourani2014}.  {\color{black} In this example, we use the MNIST dataset for image classification for the FL training in the 4-UAV network. Each UAV collects the MNIST data samples from the active sensors and engages in a FL image classification task. We adapt the FL convolutional neural network parameters from reference \cite{mcmahan2017communication}. 
\textbf{Table \ref{tab:simulation_parameters}} lists the urban environment and transmission parameters.}\\
The wireless sensors have $5$ cm$^2$ photovoltaic panels for harvesting solar energy. Given that the energy arrival rate in a sunny day is $15$ mW/cm$^2$ \cite{grossi2021energy}, each sensor harvests energy at a rate of $7$5mW. A sensor will transmit using the harvest-then-use energy harvesting scheme if it harvests at least $E_T = 100$ mJ and is covered by a UAV.\\ 
Each UAV can take two beamwidths, namely, $\theta = 41^{\circ}$ and $47.39^{\circ}$, which correspond to coverage radii $200$ m and $250$ m, respectively. Thus, the bi-objective MAB problem has $2^4 = 16$ arms, arm 1 has the configuration $\boldsymbol{\theta}_1=\{41^{\circ},41^{\circ},41^{\circ},41^{\circ}\}$, arm 2 is  $\boldsymbol{\theta}_2=\{41^{\circ},41^{\circ},41^{\circ},47.39^{\circ}\} $ and so on. {\color{black} \textbf{Fig. \ref{fig:arm_configurations}} illustrates the simulated network for four different arms, i.e., beamwidth configurations.} \\
Let $\boldsymbol{\mu}_1, \dots, \boldsymbol{\mu}_{16}$ be the normalized average reward vectors of the 16 arms. A reward vector $\boldsymbol{\mu}_{k}$ dominates vector $\boldsymbol{\mu}_\alpha$ if $\boldsymbol{\mu}_{k}$ is not worse than $\boldsymbol{\mu}_\alpha$ in all objectives and $\boldsymbol{\mu}_{k}$ is strictly better than $\boldsymbol{\mu}_{\alpha}$ in at least one objective. The reward vectors that are not dominated by any other reward vector are optimal and they form the Pareto front. For benchmarking, we find the optimal reward set $\mathcal{O}^{*}$ and the corresponding arm set $\mathcal{A}^{*}$ {\color{black} by finding the empirical normalized average rewards of all arms over $5\times 10^6$ iterations. Then, the set of non-dominated arms is found by exhaustive search, i.e., by comparing the arms' reward vectors to find the dominance relations. The same Pareto front was obtained using the Non-Dominated Sorting Genetic Algorithm II (NSGA-II)}. 
\textbf{Fig.~\ref{fig:reward_Exhaustive_ParetoUCB}} shows the normalized average reward for all the arms, the set of non-dominated average rewards $\mathcal{O}^{*}$ corresponds to the optimal arm set $\mathcal{A}^{*} = \{1, 3, 7, 8, 15, 16\}$. {\color{black} To interpret this result in terms of the network parameters, \textbf{Table \ref{tab:arm_coverage_delay}} transforms the normalized average reward of each arm to the corresponding average number of covered sensors $\bar{N}_{\boldsymbol{\theta}}$ and the average maximum delay denoted by $\bar{T}_{\max}$. Clearly, arms 1 and 16 are optimal because they minimize the maximum delay and maximize the coverage, respectively. Arms 3, 7 ,8, and 15 achieve a tradeoff between the two objectives. However, arm 9, for example, is not optimal since arm 7 achieves a higher coverage and less delay.} 
\begin{figure}[ht]
  \centering
  \begin{subfigure}[b]{0.23\textwidth}
     \centering
         \includegraphics[height=3.2cm]{ 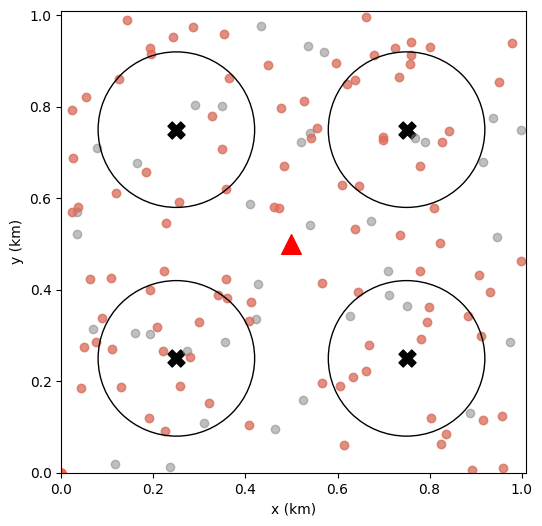}
     \caption{\footnotesize{Arm 1 corresponds to the configuration $\boldsymbol{\theta}_1=\{41^{\circ},41^{\circ},41^{\circ},41^{\circ}\}$ }}
  \end{subfigure}
     \hfill
  \begin{subfigure}[b]{0.23\textwidth}
     \centering
       \includegraphics[height=3.2cm]{ 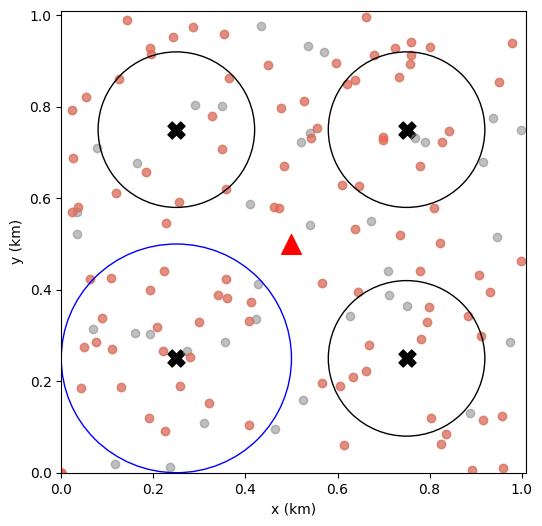}
       \caption{\footnotesize{Arm 2 corresponds to the configuration $\boldsymbol{\theta}_2= \{41^{\circ},41^{\circ},41^{\circ},47.39^{\circ}\}$} }
     \end{subfigure}
          \hfill
  \begin{subfigure}[b]{0.23\textwidth}
     \centering
       \includegraphics[height=3.2cm]{ 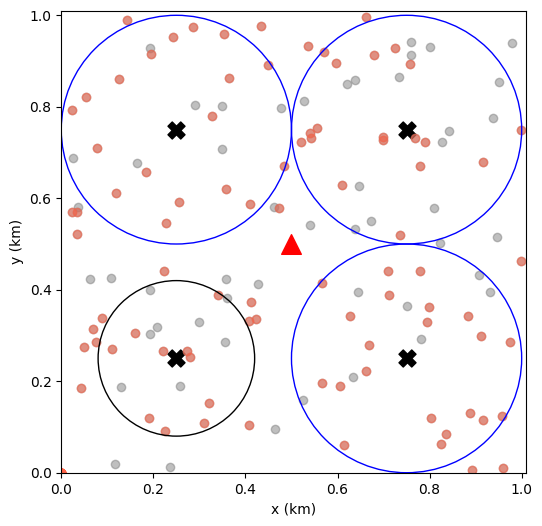}
       \caption{\footnotesize{Arm 15 corresponds to the configuration $\boldsymbol{\theta}_{15}=\{47.39^{\circ},47.39^{\circ},47.39^{\circ},41^{\circ}\}$} }
     \end{subfigure}
          \hfill
  \begin{subfigure}[b]{0.23\textwidth}
     \centering
       \includegraphics[height=3.2cm]{ 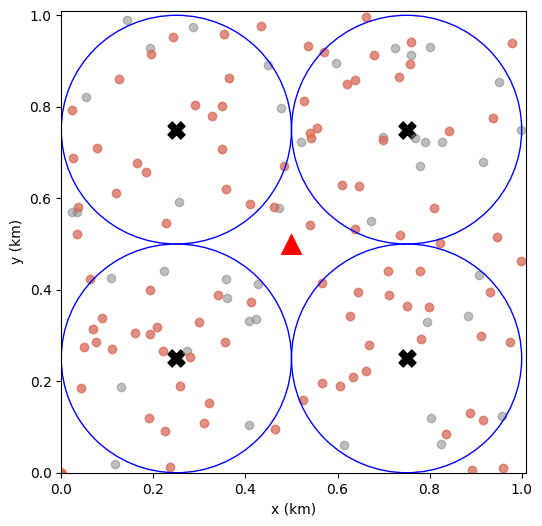}
       \caption{\footnotesize{Arm 16 corresponds to the configuration $\boldsymbol{\theta}_{16}=\{47.39^{\circ},47.39^{\circ},47.39^{\circ},47.39^{\circ}\}$} }
     \end{subfigure}
 \caption{\small{An illustration of the network, the red triangle represents the ground server, and the circles represent the coverage area of each of the UAVs. The red and gray dots represent the active and inactive sensors, respectively.}}
     \label{fig:arm_configurations}
\end{figure}
\begin{table}[]
    \centering
        \caption{Arm's coverage and maximum delay}
       \begin{tabularx}{\textwidth}{|p{14pt}|X|X|X|X|X|X|X|X|X|X|X|X|X|X|X|X|}   
       \hline
        Arm & \cellcolor{blue!25}1 & 2 & \cellcolor{blue!25}3 & 4 & 5 & 6 & \cellcolor{blue!25}7 & \cellcolor{blue!25}8 & 9 & 10 & 11 & 12 & 13 & 14 & \cellcolor{blue!25}15 & \cellcolor{blue!25}16  \\ \hline
        $\bar{N}_{\boldsymbol{\theta}}$& \cellcolor{blue!25}55.5 & 62.6  & \cellcolor{blue!25}63.1 & 71.1 & 62.1 & 69.1 & \cellcolor{blue!25}68.8 & \cellcolor{blue!25}76.1 & 61.3 & 68.4 & 68.8 & 76.8 &  67.5 & 74.4 & \cellcolor{blue!25}75.4 & \cellcolor{blue!25}83.1 \\  \hline
        $\bar{T}_{\max}$  & \cellcolor{blue!25}1.13  & 1.31  & \cellcolor{blue!25} 1.16 & 1.31 & 1.21  & 1.32 &\cellcolor{blue!25} 1.22  & \cellcolor{blue!25}1.32 &1.28 & 1.35  & 1.30 & 1.35 &1.30 & 1.35 & \cellcolor{blue!25} 1.30 & \cellcolor{blue!25} 1.35 \\ \hline
    \end{tabularx}
    \label{tab:arm_coverage_delay}
\end{table}
\textbf{Fig.~\ref{fig:percentage_played}} shows the fraction of time each arm is played in $5 \times 10^6$ rounds. Pareto UCB1 pulls the optimal arms fairly. For the scalarized bi-objective UCB1 algorithm, we use 11 uniformly spaced weights, $w^j \in \{[1,0], [0.9,0.1],\dots, [0,1]\}$. The LinS method cannot find all the optimal arms in the non-convex Pareto front because there is no set of weights that describes all optimal reward vectors. Consequently, the algorithm exploits arms 1 and 16 significantly more frequently than the other optimal arms. {\color{black} The ChebS performs better than LinS in exploiting the optimal arms as it can find the optimal arms in a non-convex Pareto front. This fact also appears in terms of the unfairness regret \textbf{Fig.~\ref{fig:fariness_regret}}. }
\textbf{Fig.~\ref{fig:cumulative_regret}} shows the average regret. The Pareto UCB1 algorithm has a low average regret as it finds all optimal arms and plays them fairly. In contrast, the LinS algorithm incurs a large average regret as it fails to find all the optimal arms. {\color{black} ChebS has a lower average regret than LinS as it performs better in finding the optimal arms.} \\
\begin{figure}[ht]
  \centering
  \begin{subfigure}[b]{0.49\textwidth}
     \centering
         \includegraphics[width=3.1in]{ 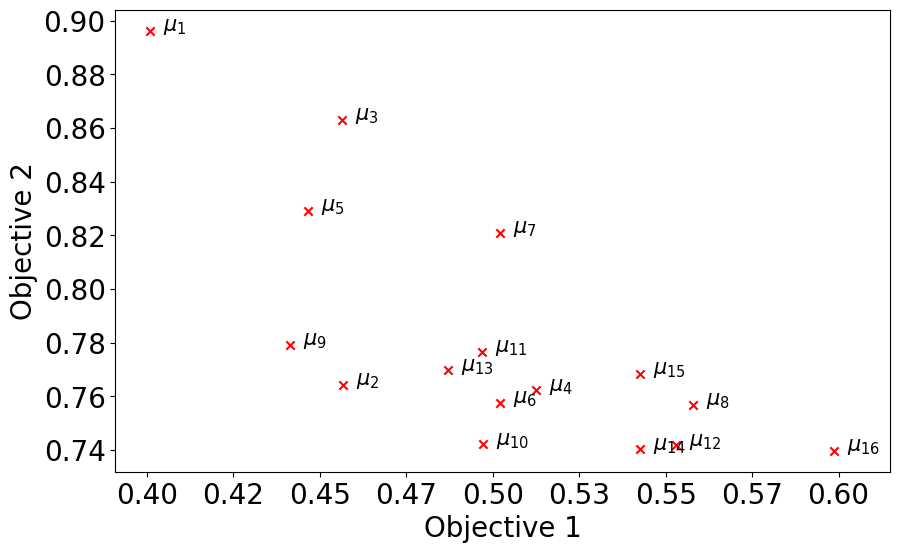}
     \caption{\footnotesize {The average reward for all arms.}} \label{fig:reward_Exhaustive_ParetoUCB}
  \end{subfigure}
     \hfill
  \begin{subfigure}[b]{0.49\textwidth}
     \centering
       \includegraphics[width=3.1in]{ 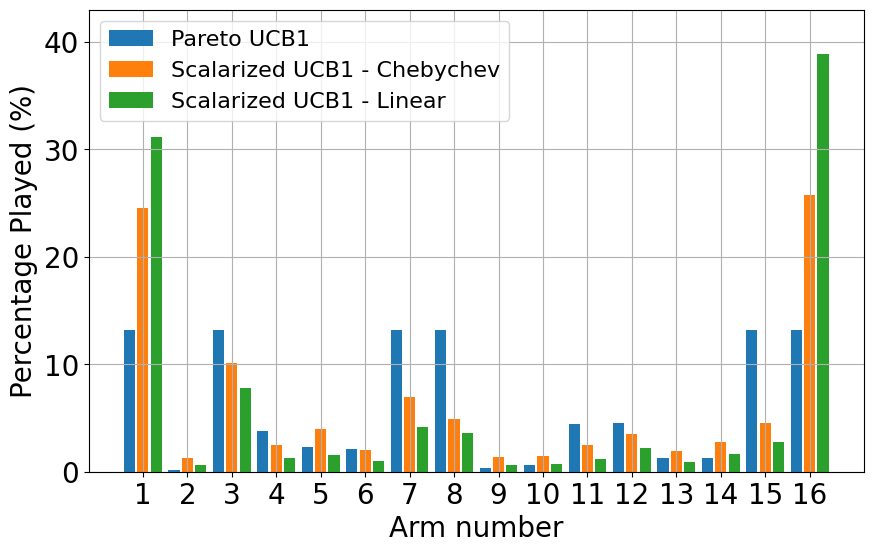}
       \caption{\footnotesize{The percentage of playing each arm.}}
         \label{fig:percentage_played}
     \end{subfigure}
  \begin{subfigure}[b]{0.49\textwidth}
   \centering
     \includegraphics[width=3.1in]{ 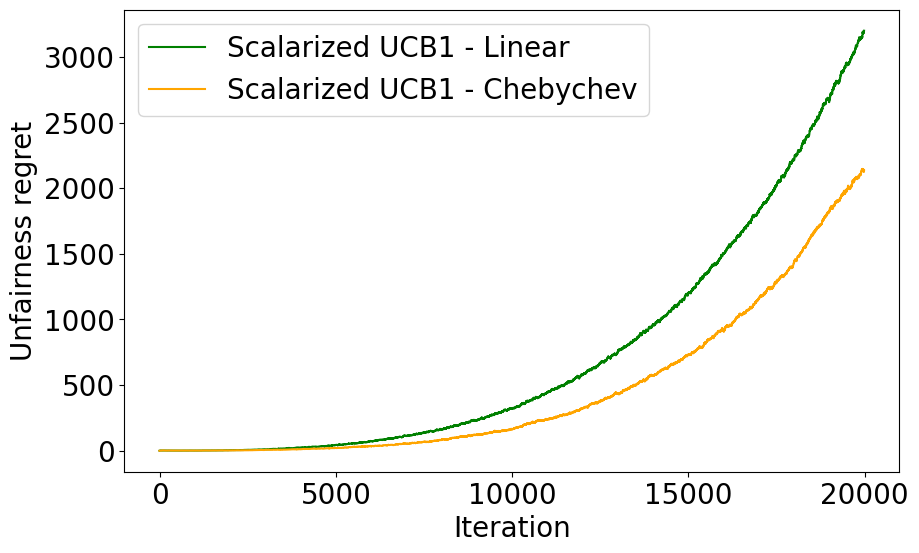}
     \caption{\footnotesize{The unfairness regret of the linearly scalarized UCB1.}}
       \label{fig:fariness_regret}
   \end{subfigure}
        \hfill
  \begin{subfigure}[b]{0.49\textwidth}
     \centering
         \includegraphics[width=3.1in]{ 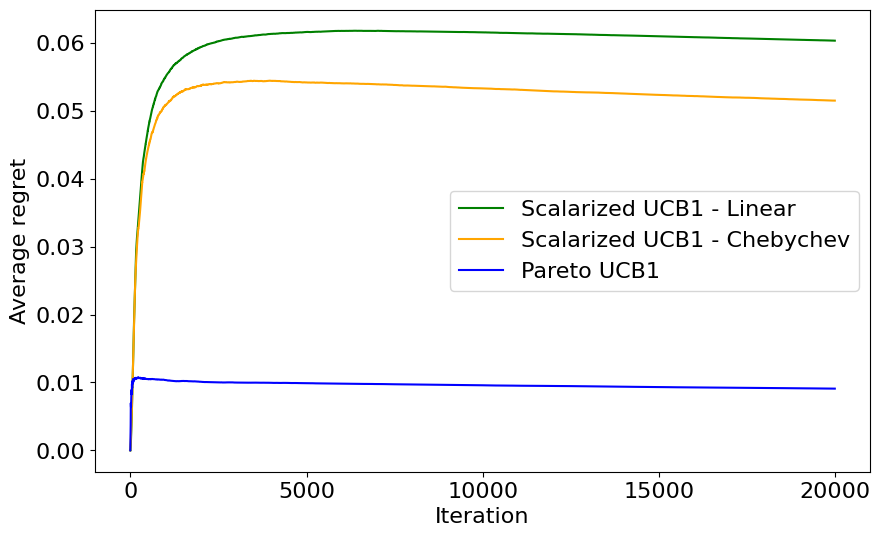}
     \caption{\footnotesize {The average regret.}}
       \label{fig:cumulative_regret}
  \end{subfigure}
 \caption{The performance of the 4-UAV network}
     \label{fig:regrets}
\end{figure}
\begin{figure}
    \centering
    \includegraphics[scale=0.35]{ 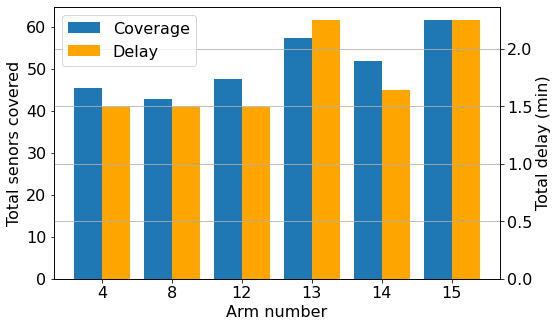}
    \caption{The sensor coverage and total delay at the optimal beamwidth configurations. }
    \label{fig:optimal_coverage_and_delay}
\end{figure}
The multi-objective optimization problem returns a set of optimal arms. However, the server selects one of them as the UAVs' beamwidths based on its preference and the application. To show this, we consider a 4-UAV network that is similar to the one described above, with the exception that the sensors' locations follow an inhomogeneous Poisson distribution with density $\lambda(x,y)= 150(x^2 +y^2)$ sensors/km$^2$. Among the 16 arms, six are optimal. \textbf{Fig.~\ref{fig:optimal_coverage_and_delay}} shows the optimal arms with their coverage and maximum delay. As expected, the circle packing placement (arm 16) maximizes coverage; nevertheless, it might not fit the FL process as it prolongs the delay.
Generally, using any of these 6 arms is optimal as opposed to the sub-optimal arms.\\
\begin{figure}[ht]
  \centering
  \begin{subfigure}[b]{0.30\textwidth}
     \centering
         \includegraphics[height=4cm]{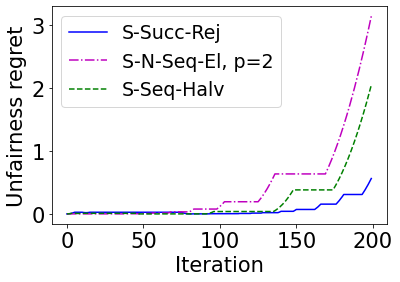}
     \caption{\footnotesize {The unfairness regret when  the arms' energy cost is considered.}}
       \label{fig:unfairness_nEq_cost_iter200}
  \end{subfigure}
     \hfill
  \begin{subfigure}[b]{0.30\textwidth}
     \centering
       \includegraphics[scale=0.4]{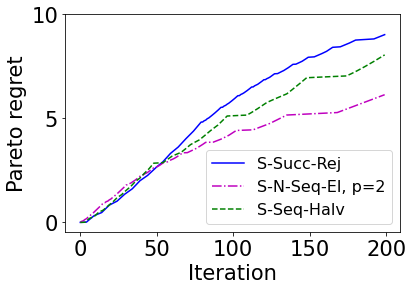}
       \caption{\footnotesize{The Pareto regret when the arms' energy cost is considered.} }
         \label{fig:pareto_regret_nEq_cost_iter200}
     \end{subfigure}
          \hfill
  \begin{subfigure}[b]{0.30\textwidth}
     \centering
       \includegraphics[height=4cm]{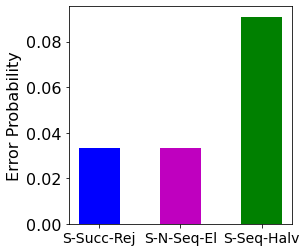}
       \caption{\footnotesize{The error probability when  the arms' energy cost is considered.} }
         \label{fig:error_nEq_cost_iter200}
     \end{subfigure}
          \hfill
  \begin{subfigure}[b]{0.30\textwidth}
     \centering
         \includegraphics[height=4cm]{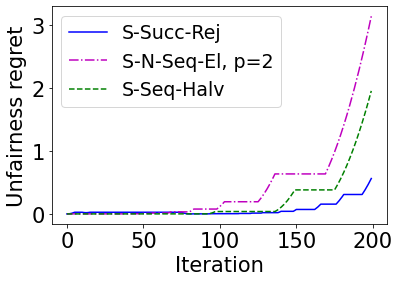}
     \caption{\footnotesize {The unfairness regret for arms with unit cost.}}
       \label{fig:unfairness_cost1_iter200}
  \end{subfigure}
     \hfill
  \begin{subfigure}[b]{0.30\textwidth}
     \centering
       \includegraphics[scale=0.4]{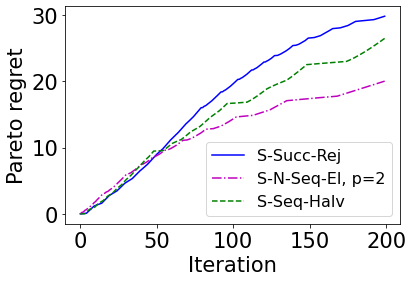}
       \caption{\footnotesize{The Pareto regret for arms with unit cost.} }
         \label{fig:pareto_regret_cost1_iter200}
     \end{subfigure}
          \hfill
  \begin{subfigure}[b]{0.30\textwidth}
     \centering
       \includegraphics[height=4cm]{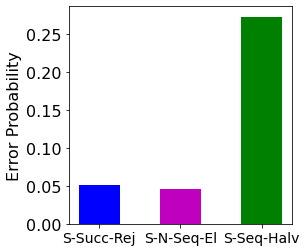}
       \caption{\footnotesize{The error probability for arms with unit cost.} }
         \label{fig:error_cost1_iter200}
     \end{subfigure}
          \hfill
 \caption{The unfairness regret, average regret, and error probability for different instances the S-C-G-Seq-El algorithm for with and without considering the energy cost.}
     \label{fig:regrets_best_arm}
\end{figure}
The next example shows the performance of the S-C-G-Seq-El algorithm using the network of the first numerical example and linear scalarization. The S-C-G-Seq-El algorithm starts by estimating the budget based on the maximum energy consumed by the UAVs per global round when the UAVs take the largest beamwidths. Then, the optimal arms are obtained by finding the best arm for each scalarization function. The performance analysis is here more complicated than the single objective case \cite{audibert2010best, shahrampour2017sequential} because there are multiple performance measures.\\
In \textbf{Fig.~\ref{fig:unfairness_nEq_cost_iter200} }-\textbf{Fig.~\ref{fig:error_nEq_cost_iter200}}, we aim to find the arms that maximize the expected reward to the expected cost ratio for three instances of the S-C-G-Seq-El, S-Succ-Rej, S-N-Seq-EL with $p=2$, and S-Seq-Halv using linear scalarization. The prefix (S-) stands for scalarization to indicate that this is a bi-objective problem. Here, the optimal arms are arms 1, 3 and 15. \textbf{Fig.~\ref{fig:unfairness_nEq_cost_iter200}} shows the unfairness regret. Although both algorithms eliminate one arm every round, the S-Succ-Rej has a lower unfairness regret than S-N-Seq-EL. This is because, in the first rounds when most arms are still available, S-Succ-Rej plays the arms longer than S-N-Seq-El, leaving less budget for the following rounds which results in higher fairness. \textbf{Fig.~\ref{fig:pareto_regret_nEq_cost_iter200}} shows the instantaneous Pareto regret given by $\Delta_k^j = \frac{f^j(\boldsymbol{\bar{X}}_{k^{*}})}{\bar{c}_{k^{*}}} - \frac{f^j(\boldsymbol{\bar{X}}_k)}{\bar{c}_k}$, where $k^{*}$ is the optimal arm under $f^j$, averaged on all scalarization functions. The S-N-Seq-El with $p=2$ has the lowest Pareto regret, followed by Seq-Halv which implies that eliminating half of the arms in each round does not necessarily result in a higher Pareto regret.  \textbf{Fig.~\ref{fig:error_nEq_cost_iter200}} shows the probability of discarding the optimal arm for a given scalarization function, averaged over all scalarization functions. The three algorithms are effective in finding the optimal arms as they have a low error probability. \\
{\color{black} Our proposed algorithm considers the arms' energy cost to ensure that the recommended UAV configuration is energy efficient. But the algorithm can be used to find the arms that maximize the expected reward without considering the cost, i.e., when all arms have a unitary cost.
\textbf{Fig.~\ref{fig:unfairness_cost1_iter200}} -\textbf{Fig.~\ref{fig:error_cost1_iter200}} show the results for the latter case. The discussion here is similar to the previous case except for higher Pareto regret and error probability. This is because when the arms' cost was considered, arm 1 had a much higher ratio of expected reward to expected energy cost compared to the others arms, and the algorithm was capable of finding it with a low error probability despite the limited budget. However, for the case of unitary cost, the sub-optimality gap was smaller, leading to more errors in finding the optimal arm for a given scalarization function $f^j$. Still, the low Pareto regret shows that the reward of the recommended arm is close to the optimal one.}
\section{Conclusion}
\label{sec:conclusion}
In this paper, the performance of FL in UAV-enhanced wireless networks was optimized under uncertainty about channels' quality, sensors' distribution, and transmission probability. Each UAV collects data from the sensors in its coverage regions and operates as an FL client. As the network parameters significantly affect the FL performance, we jointly optimize two conflicting objectives, maximizing the sensor coverage and minimizing the FL convergence time. We proposed a solution based on MO-MAB with two variations, Pareto UCB1, and linear multi-objective UCB1. Besides, we developed an efficient solution, S-C-G-Seq-El, based on the best arm identification concept. We define the best arm as one that maximizes the expected reward to the expected energy-cost ratio for an energy-efficient solution. Numerical results show that Pareto UCB1 performs better than the linearly-scalarized UCB1 algorithm in finding the optimal arms and playing them almost equally frequently. Then, we investigated the performance of three instances of the S-C-G-Seq-El algorithm. Although S-Seq-Halv removes half of the arms in every round, its performance concerning unfairness and Parto regret is comparable to S-Succ-Rej and S-N-Seq-El, which eliminate one arm per round. Our proposed solution has a wide range of applications in communication networks, as most optimization problems in such a network are intrinsically multi-objective. An example is client selection in FL to jointly optimize the convergence time, accuracy, and energy consumption. 
\section{Appendix}
\subsection{Proof of Proposition \ref{prop:T_uav_ground}}
\label{sec:appendix_T_uav_ground}
Let $\gamma_{i,\text{G}} = \frac{P_{\text{t,G}} G_{i,\text{G}}}{L_{i,\text{G}} \sigma^2_{\text{G}}}$. The transmission rate of the UAV or the server is
\begin{equation}
    r_{i,\text{G}} = B_{\text{G}} \log_2 \left( 1 + \gamma_{i,\text{G}} |g_{i,\text{G}}|^2 \right).
\label{eq:proof_rate_ground}
\end{equation}
The time required to transmit a model of size $S_{\text{M}}$ is given by
\begin{equation}
    T^{\text{cm}}_{i,\text{G}} = \frac{S_{\text{M}}}{r_{i,\text{G}}}.
\label{eq:proof_Tcm_ground}
\end{equation}
To obtain the distribution of the UAV-server communication time, we first find the distribution of the UAV-server data rate given in \eqref{eq:proof_rate_ground} and then use \eqref{eq:proof_Tcm_ground} to find the distribution of the communication time.\\
When $g_{i,\text{G}}$ is Ricean distributed, then $\Psi_{i,\text{G}} =|g_{i,\text{G}}|^2$ follows the non-central $\chi^2$  distribution with two degrees of freedom. Its PDF and CDF respectively are \cite{john2008digital}
\begin{equation}
    f_X(x) = \frac{1}{2 \xi^2} e^{-\frac{k^2+x}{2 \xi^2}} I_0\left(\frac{k}{\xi^2} \sqrt{x} \right), \quad  \text{and}  \quad  
    F_X(x) = 1 - Q_1 \left( \frac{k}{\xi},\frac{\sqrt{x}}{\xi}  \right), \quad x>0,
    \label{eq:chi_pdf_cdf}
\end{equation}
where $\xi$ and $k$ are the parameters of the $\chi^2$ distribution. $I_0(x)$ is the modified Bessel function of the first kind and order zero, i.e., $    I_0(x) =\frac{1}{2\pi} \int_0^{2 \pi} e^{x \cos \phi} d\phi$, 
%
%
and $Q_1(.,.)$ is the Marcum $Q$ function, given by $Q_1(a,b) = \int_b^\infty x e^{- \frac{a^2 +x^2}{2}} I_0(ax) dx$.
%
%
Thus, the distribution of $\Psi_i$ yields 
\begin{equation}
    f_{\Psi_{i,\text{G}} }(\psi_{i,\text{G}}) = \frac{1}{2 \xi_{\text{G}}^2} e^{-\frac{k_{i,\text{G}}^2+\psi_{i,\text{G}}}{2 \xi_{\text{G}}^2}} I_0\left(\frac{k_{i,\text{G}}}{\xi_{\text{G}}^2} \sqrt{\psi_{i,\text{G}}} \right), \qquad  \psi_{i,\text{G}}>0
    \label{eq:fading_pdf}
\end{equation}
Now, the distribution of the data rate $r_{i,\text{G}}$ in \eqref{eq:proof_rate_ground} can be obtained using the transformation of random variables. Besides, $\mathbb{P}[r_{i,\text{G}}<0]=0$ because $\log_2(1+x) >0$ for $x>0$,
\begin{equation}
    \mathbb{P}[r_{i,\text{G}} \le \rho_{i,\text{G}}] = \mathbb{P} \left[ B_{\text{G}} \ln \left( 1 + \gamma_i \psi_i \right)  \le \rho_{i,\text{G}} \right] 
    = \mathbb{P}\left[  \psi_{i,\text{G}} \le \frac{e^{\frac{\rho_{i,\text{G}}}{B_{\text{G}}}} -1}{\gamma_{i,\text{G}}} \right]. 
    \label{eq:rate_cdf1}
\end{equation}
Now, substituting \eqref{eq:rate_cdf1} in $F_X(x)$ from \eqref{eq:chi_pdf_cdf} results in the CDF of the transmission rate
\begin{equation}
   F_{r_{i,\text{G}}}(\rho_{i,\text{G}})
   = F_{\Psi_{i,\text{G}}} \left( \frac{e^{\frac{\rho_{i,\text{G}} }{B_{\text{G}}}} -1 }{\gamma_{i,\text{G}}} \right)
   = 1 - Q_1 \left( \frac{k_{i,\text{G}}}{\xi_{\text{G}}}, \frac{\sqrt{e^{\frac{\rho_{i,\text{G}}}{B_{\text{G}}}} - 1}}{\xi_{\text{G}} \sqrt{\gamma_{i,\text{G}}}} \right).
\label{eq:capacity_cdf_chi}
\end{equation}
The CDF can be rewritten as:
\begin{equation}
   F_{r_{i,\text{G}}}(\rho_{i,\text{G}}) =   
   1 - Q_1 \left( \frac{k_{i,\text{G}}}{\xi_{\text{G}}}, \frac{\sqrt{e^{\frac{\rho_{i,\text{G}}}{B_{\text{G}}}} - 1}}{\xi_{\text{G}} \sqrt{\gamma_{i,\text{G}}}} \right)
    = 1 - \int^{\infty}_{\frac{\sqrt{e^\frac{\rho_{i,\text{G}}}{B_{\text{G}}}-1}}
    {\xi_{\text{G}} \sqrt{\gamma_{i,\text{G}}}}} x e^{- \frac{\frac{k_{i,\text{G}}^2}{\xi_{\text{G}}^2} +x^2}{2}} I_0 \left( \frac{k_{i,\text{G}} x}{\xi_{\text{G}}} \right) dx.
\label{eq:cdf_capacity_ground}
\end{equation}
To find the PDF of $r_{i,\text{G}}$, we must find the derivative of the CDF given in \eqref{eq:cdf_capacity_ground}. By applying the Leibniz rule \cite{papoulis02} to \eqref{eq:cdf_capacity_ground}, we arrive at
\begin{equation}
    f_{r_{i,\text{G}}}(\rho_{i,\text{G}}) = \frac{1}{2B_{\text{G}}\gamma_{i,\text{G}} \xi_{\text{G}}^2}
    e^{ \frac{k_{i,\text{G}}^2 \gamma_{i,\text{G}} + e^\frac{\rho_{i,\text{G}}}{B_{\text{G}}} -1}{2\gamma_{i,\text{G}} \xi_{\text{G}}^2} }
    I_0 \left(\frac{k_{i,\text{G}} \sqrt{e^\frac{\rho_{i,\text{G}}}{B_{\text{G}}} -1}}{\xi_{\text{G}}^2 \sqrt{\gamma_{i,\text{G}}}} \right).
\label{eq:pdf_rate_ground}
\end{equation}
The final step is to obtain the distribution of the UAV-server communication time using \eqref{eq:proof_Tcm_ground} by transforming the random variable. Let $y = g(x) = \frac{a}{x}$ with $x, a \ge 0$. The PDF of $y$ yields $f_Y(y) = \frac{a}{y^2} f_X \left(\frac{a}{y} \right)$ \cite{papoulis02}. Thus, the PDF of the communication time is given by
\begin{equation}
    f_{T_{i,\text{G}}}(t_{i,\text{G}}) = \frac{S_\text{M}}{2B_{\text{G}} \gamma_{i,\text{G}} \xi_{\text{G}}^2 t_{i,\text{G}}^2} e^{ \frac{k_1^2 \gamma_{i,\text{G}} + e^{\left(\frac{S_\text{M}}{B_{\text{G}} t_{i,\text{G}}} \right)} -1}{2\gamma_{i,\text{G}} \xi_{\text{G}}^2} }I_0 \left( \frac{k_{i,\text{G}} \sqrt{e^\frac{S_\text{M}}{B_{\text{G}} t_{i,\text{G}}} -1}}{\xi_{\text{G}}^2 \sqrt{\gamma_{i,\text{G}}}}\right),
\end{equation}
which completes the proof. 
\subsection{Proof of Proposition \ref{prop:max_distance}}
\label{sec:appendix_max_distance}
The sensors are distributed according to a  Poisson distribution. Thus, their pairwise distance follows an exponential distribution, i.e.,
$        f_{Z_i}(z_i) = \lambda \Omega_{i,\theta_i} e^{\left(\lambda \Omega_{i,\theta_i} \right)z_i}$.\\
The maximum distance between the sensor and a UAV is thus the maximum order statistics of the exponential distribution. The maximum order statistics of $n$ independent samples of a continuous random variable is \cite{arnold2008first} $f_{X_{(n:n)}}(x) = n \left[F_X(x)\right]^{n-1} f_X(x)$. Thus, the distribution of the maximum sensor distance within the coverage area of UAV $i$, $\Omega_{i,\theta_{i}}$, results from substituting $f_{Z_i}(z_i)$ in $f_{X_{(n:n)}}(x)$. That is,  
\begin{equation}
f_{{\bar{Z}_{i(\bar{N}_{i,\theta_i}:\bar{N}_{i,\theta_i})}}} (z_i) 
  = \lambda \bar{N}_{i,\theta_i} \Omega_{i,\theta_i} \left[1 - e^{\left(\lambda \Omega_{i,\theta_i} \right)z_i}\right]^{\bar{N}_{i,\theta_i}-1}  \Omega_{i,\theta_i} e^{\left(\lambda \Omega_{i,\theta_{i}} \right)z_i},
\end{equation}
where $\bar{N}_{i,\theta_i} = \lambda \Omega_{i,\theta_i}$ represents the average number of active sensors in the coverage disk of UAV $i$. The average maximum distance is then
\begin{equation}
    \bar{Z}_{i,\theta_i} = \int_0^\infty z_i f_{{Z_i}_{\bar{N}_{i,\theta_i}:\bar{N}_{i,\theta_i}}} (z_i) dz_i = \int_0^\infty z_i \lambda \bar{N}_{i,\theta_i} \Omega_{i,\theta_i} e^{\left(\lambda \Omega_{i,\theta_i} \right)z_i} \left[1 - e^{\left(\lambda \Omega_{i,\theta_i} \right)z_i}\right]^{\bar{N}_{i,\theta_i}-1} dz_i.
\end{equation}
%
\subsection{Proof of Proposition~\ref{prop:t_total}}
\label{sec:appendix_t_total}
The computation time $T^{\text{cp}}_{i,\theta}$ is a discrete random variable while the communication time $T^{\text{cm}}_{i,\theta}$ is continuous. They are dependent as both are a function of $\theta_i$. Thus, the cumulative density function (CDF) of their sum $T^{\text{total}}_{i,\theta} = T^{\text{cp}}_{i,\theta} + T^{\text{cm}}_{i,\theta}$ is derived from the definition of the CDF, i.e.,
\begin{align}
     F_{T_{i,\theta_i}^{\text{total}}} \left(t_{i,\theta_i}^{\text{total}} \right) &= \mathbb{P}\left[T_{i,\theta_i}^{\text{total}} \le t_{i,\theta_i}^{\text{total}} \right]  =\mathbb{P}\left[t_{i,\theta_i}^{\text{cp}} + t_{i,\theta_i}^{\text{cm}} \le t_{i,\theta_i}^{\text{total}}  \right] \\ \nonumber
     &=\sum_{\tau \in \bar{N}_{i,\theta_i}} \mathbb{P} \left[ t_{i,\theta_i}^{\text{cm}}  \le t_{i,\theta_i}^{\text{total}} - \tau |T_{i,\theta_i}^{\text{cp}}  = \tau   \right]
      \mathbb{P}\left[T_{i,\theta_i}^{\text{cp}} = \tau \left(I(\kappa) \frac{CD}{f_{\text{CPU}}}\right) \right].
\end{align}
\subsection{Proof of Theorem \ref{theorem:error_propability}}
\label{sec:appendix_theorem}
Each round $\rho$ starts with a set $\boldsymbol{G}_\rho$ of $g_\rho$ available arms. At the end of the round, the algorithm removes the set $\boldsymbol{B}_\rho$ of the worst $b_\rho$ arms. Accordingly, an error occurs when the algorithm recommends a sub-optimal arm by the end of the $R$ rounds \cite{shahrampour2017sequential}, i.e.,
\begin{equation}
    \mathbb{P}\left[\boldsymbol{G}_{R+1} \neq \{k^{*}\}\right] = \mathbb{P}[k^{*} \in \cup_{\rho=1}^R \boldsymbol{B}_\rho] =\sum_{\rho=1}^R \sum_{G_\rho} \mathbb{P}[k^{*} \in \boldsymbol{B}_\rho| \boldsymbol{G}_\rho = G_\rho] \mathbb{P}[ \boldsymbol{G}_\rho = G_\rho].
\end{equation}
The last equality holds because the sets of eliminated arms in different rounds are disjoint.
Following the notation in \cite{audibert2010best}, we use  $(k) \in\{1, \dots,K \}$ to denote $k$-th best arm with the ties broken arbitrarily, thus $\frac{\mu^X_{(1)}}{\mu^c_{(1)}} \ge  \frac{\mu^X_{(2)}}{\mu^c_{(2)}} \ge \dots \ge  \frac{\mu^X_{(K)}}{\mu^c_{(K)}}$.
We also define the \textit{sub-optimality gap} as
$\Delta_k =  \frac{\mu^X_{k^{*}}}{\mu^c_{k^{*}}} -  \frac{\mu^X_{k}}{\mu^c_{k}}$. Thus, $\Delta_{(1)} = \Delta_{(2)} \le  \Delta_{(3)} \le \dots \le  \Delta_{(K)}$.\\
To find the upper bound on the error probability, we consider the worst case. It occurs when the round starts with the set of $g_\rho$ best arms so that $G_\rho = \{(1),(2),\dots, (g_\rho)\}$, and the optimal arm is eliminated with $B_\rho= \{(g_\rho - b_\rho+1), \dots, (g_\rho)\}$. Formally,
\begin{equation}
    \frac{\bar{X}_{k^{*},n_\rho}}{\bar{c}_{k^{*},n_\rho}} \le  \underset{k\in \{ (g_\rho-b_\rho+1),\dots,(g_\rho-1),(g_\rho)\}}{\max} \frac{\bar{X}_{(k),n_\rho}}{\bar{c}_{(k),n_\rho}}.
    \label{eq:error_event}
\end{equation}
For any other $G_\rho$ the worst arm cannot be better than $(g_\rho - b_\rho+1)$ \cite{shahrampour2017sequential}. Therefore, $ \mathbb{P}(\boldsymbol{G}_{R+1} \neq \{k^{*}\}) \le \sum_{\rho=1}^R \mathbb{P}(k^{*} \in \boldsymbol{B}_\rho| \boldsymbol{G}_\rho = \{(1),(2),\dots, (g_\rho)\})$. Using the union bound, the probability of the event in \eqref{eq:error_event} yields
\begin{equation}
  \mathbb{P}\left[\boldsymbol{G}_{R+1} \neq \{k^{*}\}\right] \le
 \sum_{\rho=1}^R \sum_{k =g_\rho-b_\rho+1}^{g_\rho} \mathbb{P}\left[ \frac{\bar{X}_{k^{*},n_\rho}}{\bar{c}_{k^{*},n_\rho}} \le  \frac{\bar{X}_{(k),n_\rho}}{\bar{c}_{(k),n_\rho}} \right].
\end{equation}
The event $\frac{\bar{X}_{k^{*},n_\rho}}{\bar{c}_{k^{*},n_\rho}} \le  \frac{\bar{X}_{(k),n_\rho}}{\bar{c}_{(k),n_\rho}}$ implies that at least one of the following two holds:
\begin{equation}
    \frac{\bar{X}_{k^{*},n_\rho}}{\bar{c}_{k^{*},n_\rho}} \le \epsilon; \quad  \frac{\bar{X}_{(k),n_\rho}}{\bar{c}_{(k),n_\rho}} \ge \epsilon.
\end{equation}
Accordingly, 
\begin{align}
  \mathbb{P}\left[\boldsymbol{G}_{R+1} \neq \{k^{*}\}\right]& \le
  \sum_{\rho=1}^R \sum_{k =g_\rho-b_\rho+1}^{g_\rho} \mathbb{P} \left[ \frac{\bar{X}_{k^{*},n_\rho}}{\bar{c}_{k^{*},n_\rho}} \le \epsilon  \right] + \mathbb{P} \left[ \frac{\bar{X}_{(k),n_\rho}}{\bar{c}_{(k),n_\rho}} \ge \epsilon  \right] \\  
   &=\sum_{\rho=1}^R \sum_{k =g_\rho-b_\rho+1}^{g_\rho} \mathbb{P} \left[ \bar{X}_{k^{*},n_\rho} - \epsilon \bar{c}_{k^{*},n_\rho} \le 0 \right] + \mathbb{P} \left[ \bar{X}_{(k),n_\rho} - \epsilon \bar{c}_{(k),n_\rho} \ge 0 \right] \\ 
    &=\sum_{\rho=1}^R \sum_{k =g_\rho-b_\rho+1}^{g_\rho}\mathbb{P} \left[-( \bar{X}_{k^{*},n_\rho}-\epsilon \bar{c}_{k^{*},n_\rho})  + (\mu^X_{k^{*}} - \epsilon \mu^c_{k^{*}})    \ge \mu^X_{k^{*}} -  \epsilon \mu^c_{k^{*}} \right]  \nonumber \\
    &   +\mathbb{P} \left[ (\bar{X}_{(k),n_\rho}-\epsilon \bar{c}_{(k),n_\rho})  -(\mu^X_{(k)} - \epsilon \mu^c_{(k)} )   \ge -\mu^X_{(k)} +  \epsilon \mu^c_{(k)} \right]. \label{eq:dependent_cost_reward}
\end{align}
Note that at every selection round, the reward and the energy cost are dependent as they both depend on $\boldsymbol{\theta}_k$ and $T_{i,\text{L}}$. However, the rewards obtained at different times are independent of each other, and so are the energy costs. Hence, the Hoeffding inequality still applies in \eqref{eq:dependent_cost_reward}.\\
\begin{equation}
    \mathbb{P}\left[\boldsymbol{G}_{R+1} \neq \{k^{*}\}\right]\le\sum_{\rho=1}^R \sum_{k =g_\rho-b_\rho+1}^{g_\rho} \exp\left(- \frac{2 n_\rho \left(\mu^X_{k^{*}} -  \epsilon \mu^c_{k^{*}} \right)^2}{ (\epsilon+1)^2} \right) +\exp\left(- \frac{2 n_\rho \left( -\mu^X_{(k)} +  \epsilon \mu^c_{(k)} \right)^2}{ (\epsilon+1)^2} \right). \label{eq:error_epsilon}
\end{equation}
By substituting the value of $\epsilon = \frac{\mu^X_{k^{*}} +\mu^X_{(k)}}{\mu^c_{k^{*}} +\mu^c_{(k)}}$ in \eqref{eq:error_epsilon} and rearranging the terms, \eqref{eq:error_epsilon} reduces to
\begin{equation}
 \mathbb{P}\left[\boldsymbol{G}_{R+1} \neq \{k^{*}\}\right] \le \sum_{\rho=1}^R \sum_{k = g_\rho-b_\rho+1}^{g_\rho} 2 \exp\left(- 2 n_\rho \left(\frac{\mu^X_{k^{*}} \mu^c_{(k)} - \mu^c_{k^{*}} \mu^X_{(k)}}{\mu^X_{k^{*}} +\mu^X_{(k)}+ \mu^c_{k^{*}}+\mu^c_{(k)} } \right)^2\right). 
\end{equation}
Let $\Xi_\rho = \underset{k\in \{(g_\rho -b_\rho+1),\dots,(g_\rho)\}}{\min}  \left(\frac{\mu^X_{k^{*}} \mu^c_{(k)} - \mu^c_{k^{*}} \mu^X_{(k)}}{\mu^X_{k^{*}} +\mu^X_{(k)}+ \mu^c_{k^{*}}+\mu^c_{(k)} } \right)$, the upper bound on the error probability yields
\begin{equation}
 \mathbb{P} \left[\boldsymbol{G}_{R+1}\neq k^{*} \right] \le \sum_{\rho=1}^R  2 b_\rho \exp\left(- 2 n_k \Xi_\rho^2\right)  \le  R \underset{\rho \in \{1,\dots,R\}}{\max}\{b_\rho\} \underset{\rho \in \{1,\dots,R\}}{\max} \left\{ 2 \exp\left(- 2 n_\rho \Xi_\rho^2 \right)\right\}.
\end{equation}
\bibliography{Journal_Paper/main}
\bibliographystyle{ieeetr}
\end{document}